\documentclass[11pt]{article}
\usepackage{amsmath,amsthm,amssymb,bm}
\usepackage{geometry}
\usepackage{graphics,epsfig,epsf,psfrag,cellspace}
\usepackage{url}
\usepackage{amsmath,amsfonts,amssymb,amsthm}
\usepackage{bbm, bm}
\usepackage{graphicx,lscape,rotate, float, afterpage}
\usepackage{movie15}  % To use GIF Image Files
\usepackage{dashrule} % To make dotted/dashed lines
\usepackage{epstopdf} % required by TeXShop and Mac to process eps figures in includegraphics environment
\usepackage{enumerate, framed}
\usepackage{colortbl}
\usepackage{color, xcolor, appendix}
\usepackage[round]{natbib}
\usepackage{movie15}  % To use GIF Image Files
\usepackage{dashrule} % To make dotted/dashed lines
\usepackage{epstopdf} % required by TeXShop and Mac to process eps figures in includegraphics environment
\RequirePackage[colorlinks,citecolor=blue,urlcolor=blue]{hyperref}

\usepackage{algorithmic} %, algorithm}
\usepackage[]{algorithm2e}
\usepackage{dsfont,setspace,subcaption}

\makeatother

\newtheorem{theorem}{Theorem}
\newtheorem{lemma}{Lemma}

\newtheorem{corollary}{Corollary}
\newtheorem{assumption}{Assumption}

\newcommand{\Var}{\mathrm{Var}}
\newcommand{\Cov}{\mathrm{Cov}}
\newcommand{\E}{\mathbb{E}}

\DeclareMathOperator*{\argmax}{\arg\!\max}

\begin{document}
% \doublespacing % <-- if you want to doublespace your document

\title{End-Cut Preference in Survival Trees}
\author{\textbf{Xiaogang Su} \\
Department of Mathematical Sciences \\
University of Texas, El Paso, TX 79968 \vspace{.1in} }
\date{}
\maketitle

\renewcommand{\abstractname}{\large Abstract \vspace{.1in}}
\begin{abstract}
{\normalsize The end-cut preference (ECP) problem, referring to the tendency to favor split points near the boundaries of a feature's range, is a well-known issue in CART \citep{Breiman:1984}. ECP may induce highly imbalanced and biased splits, obscure weak signals, and lead to tree structures that are both unstable and difficult to interpret. For survival trees, we show that ECP also arises when using greedy search to select the optimal cutoff point by maximizing the log-rank test statistic. To address this issue, we propose a smooth sigmoid surrogate (SSS) approach, in which the hard-threshold indicator function is replaced by a smooth sigmoid function. We further demonstrate, both theoretically and through numerical illustrations, that SSS provides an effective remedy for mitigating or avoiding ECP.
}
\end{abstract}

\noindent%

\section{Introduction}
\label{sec-intro}

Decision trees \citep{Morgan:1963, Breiman:1984} constitute a cornerstone of modern statistical learning. Their extension to time-to-event outcomes has led to the development of survival trees, which are designed to accommodate censored survival data. These models have become highly influential in biomedical research, where the outcome of interest is frequently the timing of events such as disease onset, relapse after treatment, hospital readmission, or attainment of developmental milestones. A key advantage of survival trees lies in their ability to generate interpretable partitioning rules that reveal the underlying structure of the data and inform the design of future studies. For comprehensive surveys and detailed bibliographies, see \citet{LeBlanc:1995} and \citet{BouHamad:2011}. Among the various construction strategies, one of the most widely adopted methods employs the logrank test statistic \citep{Mantel:1966, Peto:1972}, a classical tool in survival analysis for comparing two groups. In this framework, a node is split at the point that maximizes the difference in survival experience between the two resulting subgroups. Trees built in this way are often referred to as ``trees by goodness-of-split'' \citep{LeBlanc:1993}.

The end-cut preference (ECP) problem, first documented by \citet{Breiman:1984} in regression trees, refers to the tendency of greedy search algorithms in CART to select split points near the boundaries of a continuous predictor’s range. This phenomenon often produces highly unbalanced partitions, introduces bias, and masks weak but meaningful signals. Consequently, the resulting tree structures may be unstable and difficult to interpret. Boundary splits create terminal nodes containing only a small number of observations, which inflates variance, lowers predictive accuracy, and reduces generalizability. In addition, ECP diminishes the ability of the algorithm to detect interior structure, thereby overlooking informative cut points that could enhance model fit. Practically, trees affected by ECP tend to be deeper, more complex to visualize, and less effective for communicating results. On the theoretical side, \citet{Cattaneo:2024} demonstrated that decision trees subject to ECP fail to achieve polynomial rates of convergence in the uniform norm with non-vanishing probability, underscoring a fundamental limitation of this standard greedy approach.

ECP also arises in survival trees. In this article, we formally establish the presence of ECP in survival trees grown using the maximum logrank test statistic. In regression trees, ECP is theoretically justified through the Kolmogorov inequality and the law of the iterated logarithm (\citeauthor{Breiman:1984}, \citeyear{Breiman:1984}, Section~11.8). These arguments, however, do not extend to survival trees. Our analysis instead relies on tools from the extreme value theory of Gaussian processes combined with Edgeworth expansions. To address ECP, we adopt the smooth sigmoid surrogate (SSS) approach \citep{Su:2024}, which replaces the hard-threshold indicator function with a smooth sigmoid function, thereby yielding a continuous approximation to the logrank test statistic. This transformation turns the discrete greedy search into a smooth optimization problem. We show, both theoretically and through numerical experiments, that SSS provides an effective remedy for mitigating or avoiding ECP.

The remainder of this article is organized as follows. Section~\ref{sec-ECP} establishes that ECP arises in survival trees with probability tending to one. Section~\ref{sec-SSS} demonstrates that SSS can alleviate or eliminate ECP, depending on the choice of the shape parameter. Section~\ref{sec-numerical} also presents numerical illustrations supporting the theoretical findings. Finally, Section~\ref{sec-discussion} concludes with a discussion.

\section{End-Cut Preference with Greedy Search}
\label{sec-ECP}

Suppose that the observed survival data consis of independent and identically distributed observations \(\{(T_i,\delta_i,Z_i)\}_{i=1}^n\). Under the null setting, where ECP is examined, we assume independence between the covariate \(Z\) and the failure mechanism. Without loss of generality (WLOG), we further assume \(Z \sim \mathrm{Uniform}(0,1)\). This assumption is justified because tree-based modeling is invariant under monotone transformations of the covariate, and the probability integral transform (PIT), being monotone, can always map \(Z\) to a uniform\((0,1)\) distribution. Consequently, for any cutpoint \(c\), the population fraction of subjects in the left group is
$$\mathbb{P}(Z \leq c) = c. $$

Let \(t_1 < \cdots < t_{D_n}\) denote the distinct event times. Under a fixed censoring rate, the number of events grows linearly with the sample size, that is, \(D_n \asymp n\). We write \(a_n \asymp b_n\) to mean that there exist constants \(0<m<M<\infty\), independent of \(n\), such that \(m\,b_n \le a_n \le M\,b_n\) for all large \(n\). At each event time $t_k$, define 
$Y_k=\sum_{i=1}^n I\{T_i \geq t_k\} $
as the number of subjects at risk just prior to $t_k$, and 
$ d_k=\sum_{i=1}^n \delta_i\,I\{T_i=t_k\} $
as the number of failures occurring at $t_k$. We assume that the joint distribution of the event time and the censoring time is absolutely continuous with respect to Lebesgue measure. As a result, ties occur with probability zero, implying that the number of failures at each event time satisfies \(d_k \equiv 1\) almost surely. 
Concerning the at–risk process $Y_k$, we impose the following bulk regularity assumption. 
\begin{assumption}[Bulk regularity]\label{assump:bulk}
There exists $\rho \in (0,1)$ such that
\[
Y_k \;\asymp\; n \qquad \text{uniformly for } k \le \rho D_n .
\]
\end{assumption}
\noindent In words, during the first $\rho$ fraction of the event times, the number at risk remains of order $n$. This assumption is standard, typically stated in continuous–time form, within large–sample theory of survival analysis. 
It further implies that
\begin{equation}\label{sum-Yk-inverse}
\sum_{k=1}^{D_n} \frac{1}{Y_k}\;=\;\Theta(\log n) \mbox{~~~and~~~} 
\sum_{k=1}^{D_n} \frac{1}{Y^2_k}\;=\;\Theta(1).
\end{equation}

For a cutpoint \(c\), define
\[
Y_{kL}(c) = \sum_{i=1}^n \mathbf{1}\{T_i \ge t_k,\, Z_i \le c\},
\mbox{~~~and~~~} 
d_{kL}(c) = \sum_{i=1}^n \delta_i\,\mathbf{1}\{T_i = t_k,\, Z_i \le c\},
\]
as the number at risk and the number of failures, respectively, in the left node 
\(\{Z \le c\}\) at event time \(t_k\). We impose the following overlap regularity condition on the covariance structure of the left-group risk sets.
\begin{assumption}[Overlap regularity]\label{asm:overlap}
The covariance of the left-group risk sets satisfies the canonical overlap scaling:
\[
\sum_{i=1}^{D_n}\sum_{k=1}^{D_n} 
\Cov\!\big(Y_{iL}(c),\,Y_{kL}(c)\big) \;\asymp\; n^3\,c(1-c)
\quad \text{uniformly in } c\in(0,1).
\]
\end{assumption}
\noindent This condition is very mild and essentially follows from the law of large numbers under independent sampling of \((T_i,Z_i)\). Intuitively, the overlap between different left-group risk sets grows proportionally to the product of their sizes, yielding the canonical order \(n^3c(1-c)\). Similar overlap or covariance scaling assumptions appear in the theoretical analysis of logrank-type statistics and survival processes; see, e.g., \citet{Andersen:1982} and \citet{Lin:1993}.

Introduce
$$ b_k(c)\ :=\ \frac{Y_{kL}(c)}{Y_k},$$
the empirical left-risk fraction at cutpoint $c$. In the
\emph{bulk} of the timeline, one has
\begin{equation}\label{b-k}
b_k(c)\ \xrightarrow{a.s.}\ c
\mbox{~~~and~~~}
b_k(c)-c\ =\ O_p\!\big(n^{-1/2}\big),
\end{equation}
uniformly for $k\le \rho D_n$ and $c\in(0,1)$. 

The numerator and variance scale of the logrank statistic \citep{Mantel:1966, Peto:1972}  are defined as
\[
N(c) = \sum_{k=1}^{D_n} w_k \big\{ d_{kL}(c) - b_k(c) \big\}
\mbox{~~~and~~~} 
S^2(c) = \sum_{k=1}^{D_n} w_k^2\,V_{kL}(c),
\]
where $\{w_k\}$ is a sequence of bounded weights. 
For each event time $t_k$, let 
$\mathcal A_k = \sigma\!\big(Y_k,\,Y_{kL}(c),\,d_k\big)$ 
denote the $\sigma$-algebra generated by the risk–failure configuration at $t_k$. 
Under the no–ties assumption ($d_k \equiv 1$), the conditional variance reduces to
$$
V_{kL}(c) = \Var\!\big(d_{kL}(c)\mid \mathcal A_k\big) = b_k(c)\{1-b_k(c)\},
$$
since the finite–population correction equals one. 
Consequently,
\begin{equation}
\label{rate-Sc}
 S^2(c)\ \asymp\ n\,c(1-c) \qquad \text{uniformly in } c\in(0,1). 
\end{equation}

The logrank statistic for a cutpoint \(c\) is defined as
$$ Q(c) = \left\{\frac{N(c)}{S(c)}\right\}^2,$$
and the optimal cutpoint is given by
$$ \hat c = \arg\max_c Q(c).$$
In practice, \(\hat c\) is obtained via a greedy search over the distinct observed values of \(Z_i\), which constitutes a discrete optimization problem. For any fixed cutpoint \(c\), the statistic \(Q(c)\) is asymptotically distributed as \(\chi^2(1)\) under the null hypothesis. The maximized statistic, $\max_c Q(c),$
is referred to as the \emph{maximally selected \(\chi^2\) statistic} \citep{Miller:1982}.

\subsection{The Main Result}
\label{sec-main1}

In what follows, we show that greedy search based on maximizing the logrank statistic is prone to the end--cut preference (ECP) problem. To this end, consider the standardized statistic
\[
q(c) = \frac{N(c)}{S(c)} .
\]
so that $Q(c) = q^2(c).$ As \(c\) varies, the collection \(\{q(c): c \in (0,1)\}\) forms a mean-zero Gaussian process \citep{Miller:1982}. From an optimization perspective, maximizing \(Q(c)=q^2(c)\) is equivalent to maximizing \(|q(c)|\), which reduces to locating the extremum of \(q(c)\), since the process is symmetric about zero. Gaussian extreme-value theory \citep{Leadbetter:1983} implies that the maximizer of a mean-zero Gaussian process with continuous sample paths is, with probability tending to one, determined by its variance function, provided mild regularity conditions on the correlation structure hold. Consequently, the key step is to analyze
\[
\Var\!\big(q(c)\big).
\]

To proceed, we express
\[
q(c) = g(\bm v) = \frac{N(\bm v)}{S(\bm v)},
\]
where  
\[
\bm v = \bm v(c) = (v_j) 
= \big(d_{1L}, \ldots, d_{D_nL},\, Y_{1L}, \ldots, Y_{D_nL}\big)^\top 
\in \mathbb{R}^{2D_n}
\]
denotes the primitive random vector. This representation is natural because \(\bm v\) collects all quantities that depend on the cutoff point \(c\). Since the mapping \(g:\mathbb{R}^{2D_n}\to\mathbb{R}\) is twice continuously differentiable, we may apply the second-order variance expansion (also known as the Edgeworth expansion; see, e.g., \citeauthor{Bhattacharya:1986}, \citeyear{Bhattacharya:1986}; \citeauthor{Hall:1992}, \citeyear{Hall:1992}).

Let 
\[
\bm{\mu}=\mathbb E(\bm{v}) = c\, (1, \ldots, 1, Y_1, \ldots, Y_{D_n})^\top, 
\quad \bm{\Sigma}=\Cov(\bm{v}),
\]
and define the central moments
\[
m_{ijk}=\mathbb E[(v_i-\mu_i)(v_j-\mu_j)(v_k-\mu_k)], 
\qquad
m_{ijkl}=\mathbb E[(v_i-\mu_i)(v_j-\mu_j)(v_k-\mu_k)(v_\ell-\mu_\ell)].
\]

The multivariate version of the second–order delta/Edgeworth variance expansion \citep{Wolter:1985} is
\begin{equation}\label{eq:var-delta}
\Var\big(g(\bm{X})\big)
= T_1+T_2+T_3+o\!\left(|T_3| \right),
\end{equation}
with 
\[
T_1=\sum_{i,j} g_i\,g_j\,\Sigma_{ij},\qquad
T_2=\sum_{i,j,k} g_i\,g_{jk}\,m_{ijk},\qquad
\mbox{and~~~} T_3=\frac{1}{4}\sum_{i,j,k,\ell} g_{ij}\,g_{k\ell}\,\big(m_{ijkl}-\Sigma_{ij}\Sigma_{k\ell}\big), 
\]
where \(g_i=\partial g/\partial v_i\) and \(g_{ij}=\partial^2 g/\partial v_i\partial v_j\), all evaluated at \(\bm{v}=\bm{\mu}\). At \(\bm{\mu}\),
\[
g_{d_{kL}}=\frac{w_k}{S(\bm{\mu})},\qquad
g_{Y_{kL}}=-\,\frac{w_k}{Y_k\,S(\bm{\mu})},
\]
where 
\[
S^2(\bm{\mu})=\sum_{k=1}^{D_n} w_k^2\,b_k(\bm{\mu})\{1-b_k(\bm{\mu})\}
=\sum_{k=1}^{D_n} w_k^2\,c(1-c) \;\asymp\; n\,c(1-c).
\]
After some algebra, the Hessian terms scale uniformly in indices \(i,j\) as
\begin{equation}\label{eq:hessian-scales}
\begin{cases}
g_{d_{jL},Y_{iL}} = O\!\Big(\dfrac{1}{n\,S^3(\bm{\mu})}\Big) 
= O\!\big(n^{-5/2}[c(1-c)]^{-3/2}\big), \\[1.2ex]
g_{Y_{iL},Y_{jL}} = O\!\Big(\dfrac{1}{n^2\,S^3(\bm{\mu})}\Big) 
= O\!\big(n^{-7/2}[c(1-c)]^{-3/2}\big).
\end{cases}
\end{equation}

Our first main result is stated in Theorem~\ref{thm:main}. It shows that, under the null hypothesis and standard regularity conditions, greedy search based on the logrank statistic exhibits an end--cut preference. In particular, the maximizer of the logrank statistic lies, with high probability, in a boundary region of order \(1/n\).

\begin{theorem}\label{thm:main}
In the above-described setting, assume in addition the \emph{bulk regularity} and \emph{overlap regularity} conditions. Let \(\mathcal C_n\) denote the set of candidate cutpoints (e.g., midpoints between order statistics of \(\{Z_i\}\)). Then, with probability tending to one,
$$
\argmax_{c\in\mathcal C_n} Q(c)\ \in\ (0,O(1/n)] \,\cup\, [1-O(1/n),1).
$$
That is, the greedy search (GS) procedure selects, with high probability, a cutpoint in an end--cut region of width order \(1/n\).
\end{theorem}

\subsection{Auxiliary lemmas}
\label{sec-lemmas-GS}

To prove Theorem~\ref{thm:main}, we first establish three auxiliary lemmas, each addressing one of the three terms in \eqref{eq:var-delta}.

\begin{lemma}\label{lem:T1}
Under the no--ties assumption (\(d_k \equiv 1\)) and bulk regularity, the first--order term \(T_1\) in the Edgeworth variance expansion \eqref{eq:var-delta} satisfies
\[
T_1 \;=\; 1 - \tau,
\]
where \(\tau > 0\) with \(\tau = \Theta(\log n / n)\). Moreover, \(\tau\) is independent of \(c\).
\end{lemma}

\begin{proof}
By construction, \(T_1\) corresponds to the first--order delta-method term:
\[
T_1 = \nabla g(\bm{\mu})^\top \bm{\Sigma}\,\nabla g(\bm{\mu}) 
= \Var \!\Bigg( \sum_{i} g_i(\bm{\mu})(v_i - \mu_i) \Bigg).
\]
Expanding the linear term,
\begin{align*}
\sum_{i} g_i(\bm{\mu})(v_i - \mu_i) 
&= \sum_{k=1}^{D_n} \Big\{ g_{d_{kL}} (d_{kL} - c) + g_{Y_{kL}} (Y_{kL} - cY_k) \Big\} \\
&= \frac{1}{S(\bm{\mu})} \sum_{k=1}^{D_n} w_k 
\left\{ (d_{kL} - c) - \frac{Y_{kL} - cY_k}{Y_k} \right\} 
\quad \text{(substituting $g_{d_{kL}}$, $g_{Y_{kL}}$)} \\
&= \frac{1}{S(\bm{\mu})} \sum_{k=1}^{D_n} w_k (d_{kL} - b_k), 
\qquad \text{since } b_k = Y_{kL}/Y_k.
\end{align*}
Therefore,
\begin{equation}\label{eq:T1a}
T_1 =  \frac{1}{S^2(\bm{\mu})}\,\Var\!\Bigg(\sum_{k=1}^{D_n} w_k(d_{kL}-b_k)\Bigg)  
= \frac{\Var(N(c))}{S^2(\bm{\mu})},
\end{equation}
where \(S^2(\bm{\mu}) = c(1-c) \sum_k w_k^2.\)

\medskip
We now compute \(\Var(N(c))\) using the law of total variance, conditioning on the risk--set \(\sigma\)-field \(\mathcal A\). Given \(\mathcal A\),  
\(\E[d_{kL}\mid\mathcal A]=b_k\) and \(\Var(d_{kL}\mid\mathcal A)=b_k(1-b_k)\), with conditional independence across event times. Thus
\[
\Var(N\mid \mathcal A) = \sum_k w_k^2 b_k(1-b_k) = S^2(\bm v),
\qquad \E[N\mid \mathcal A]=0.
\]
Hence
\begin{align}
\Var(N) &= \E\!\left[S^2(\bm v)\right] 
= \sum_k w_k^2 \E\{b_k(1-b_k)\} \nonumber \\
&= \sum_k w_k^2 \Big(c(1-c) - \Var(b_k)\Big) 
\quad \text{since } \E[b_k(1-b_k)] = c(1-c)-\Var(b_k) \nonumber \\
&= S^2(\bm{\mu}) - \sum_k w_k^2 \Var(b_k). \label{eq:VarN}
\end{align}

Substituting \eqref{eq:VarN} into \eqref{eq:T1a}, and noting that
\[
\Var(b_k) \;\asymp\; \E\!\left[\frac{c(1-c)}{Y_k}\right],
\]
we obtain
\[
T_1 
= 1 - \frac{\sum_k w_k^2 \Var(b_k)}{c(1-c) \sum_k w_k^2}.
\]

\medskip
By bulk regularity, the risk sets decrease smoothly with the number of failures, and satisfy
\[
Y_k \asymp n-k+1
\qquad \text{uniformly in } k.
\]
Intuitively, after \(k-1\) failures, roughly \(n-k+1\) subjects remain at risk, up to fluctuations of smaller order. Therefore,
\[
\Var(b_k) \asymp \frac{c(1-c)}{n-k+1}.
\]
With bounded weights, \(\sum_k w_k^2 \asymp n\) and
\[
\sum_k w_k^2 \Var(b_k) \;\asymp\; c(1-c)\sum_{k=1}^{D_n} \frac{1}{n-k+1}
\;\asymp\; c(1-c)\,\log n.
\]
Thus
\[
T_1 = 1 - \tau,
\qquad
\tau = \frac{\sum_k w_k^2 \Var(b_k)}{c(1-c)\sum_k w_k^2}
= \Theta\Big(\frac{\log n}{n}\Big).
\]
Note that the cancellation of the factor \(c(1-c)\) from numerator and denominator is essential. As a result, \(\tau\) does not depend on \(c\).  This completes the proof.
\end{proof}
\noindent The conclusion of Lemma~\ref{lem:T1} is unsurprising: \(T_1\) is precisely the first–order (delta–method) contribution. For fixed \(c\), the standardized statistic \(q(\bm v)\) linearizes in the usual way, the variance scale cancels, and thus \(q(\bm v)\Rightarrow \mathcal N(0,1)\) with \(T_1=1+o(1)\). Importantly, Lemma~\ref{lem:T1} also shows that this \(o(1)\) remainder is independent of \(c\) (uniform in \(c\)), a fact used in our main results.

We next treat the second term $T_2$. 

\begin{lemma}\label{lem:T2}
Assume no ties ($d_k\equiv 1$), dense failures ($D_n\asymp n$), bounded weights ($\sup_k |w_k|<\infty$), and variance scaling 
\(
S^2(c)=\sum_{k=1}^{D_n}w_k^2\,b_k(c)\{1-b_k(c)\}\asymp n\,c(1-c)
\)
uniformly in $c\in(0,1)$. 
Assume further the summation bound $\sum_{k=1}^{D_n} Y_k^{-2}=O(1)$.
Then, uniformly for $c\in(0,1)$,
\[
T_2=\sum_{i,j,k} g_i\,g_{jk}\,m_{ijk}
\;=\; o\!\Big(\frac{1}{n\,c(1-c)}\Big).
\]
\end{lemma}

\begin{proof}
At $\bm v=\bm\mu$,
\[
g_{d_{kL}}=\frac{w_k}{S}=O(S^{-1}),\qquad 
g_{Y_{kL}}=-\,\frac{w_k}{Y_k S}=O\big((Y_k S)^{-1}\big).
\]
Differentiating $N/S$ shows the nonzero Hessian blocks satisfy the refined bounds
\[
g_{d_{jL},Y_{kL}} = O\!\Big(\frac{1}{n\,S^3}\Big),\qquad
g_{Y_{jL},Y_{kL}} = O\!\Big(\frac{1}{n^2\,S^3}\Big),
\]
which come from the $S^{-1}$ derivative acting on the \emph{averaged} variance $S^2=\sum_\ell w_\ell^2\{\cdot\}$ and yield the extra $1/n$ factors.

By Cauchy--Schwarz over $(j,k)$,
\begin{equation}\label{eq:T2-CS}
|T_2|
\;\le\;
\Big(\sum_{j,k} g_{jk}^2\Big)^{1/2}
\Big(\sum_{j,k}\big(\sum_i g_i\,m_{ijk}\big)^2\Big)^{1/2}.
\end{equation}
The refined Hessian bounds imply
$\sum_{j,k} g_{jk}^2 = O(S^{-6})$ or  $ \Big(\sum_{j,k} g_{jk}^2\Big)^{1/2}=O(S^{-3}).$
Moreover,
$$
\sum_i g_i^2 
= \sum_k \Big(g_{d_{kL}}^2 + g_{Y_{kL}}^2\Big)
\;\lesssim\; \frac{\sum_k w_k^2}{S^2} + \frac{\sum_k w_k^2/Y_k^2}{S^2}
= \Theta\!\Big(\frac{1}{c(1-c)}\Big), $$
using $S^2\asymp n\,c(1-c)$ and $\sum_k Y_k^{-2}=O(1)$. With the standard decomposition 
$\bm v-\bm\mu=\sum_{u=1}^n\xi_u$ (independent subject contributions), only within-subject third moments contribute, and one has 
$\sum_{i,j,k} m_{ijk}^2=O(n)$. Hence
\[
\sum_{j,k}\Big(\sum_i g_i\,m_{ijk}\Big)^2
\;\le\;
\Big(\sum_i g_i^2\Big)\Big(\sum_{i,j,k} m_{ijk}^2\Big)
= O\!\Big(\frac{n}{c(1-c)}\Big),
\]
so the second factor in \eqref{eq:T2-CS} is $O\!\big(\sqrt{n}/\sqrt{c(1-c)}\big)$.

Combining the two factors,
\[
|T_2|
= O\!\Big(\frac{1}{S^3}\Big)\cdot O\!\Big(\frac{\sqrt n}{\sqrt{c(1-c)}}\Big)
= O\!\Big(\frac{1}{\sqrt n\,S^2\sqrt{c(1-c)}}\Big)
= O\!\Big(\frac{1}{\sqrt n\,n\,c(1-c)}\Big)
= o\!\Big(\frac{1}{n\,c(1-c)}\Big),
\]
uniformly in $c\in(0,1)$.
\end{proof}

\noindent 
Finally, we turn to the third term \(T_3\). As shown below, the third–moment contribution \(T_2\) is of smaller order (uniformly in \(c\)) and is therefore negligible relative to \(T_3\).

\begin{lemma} \label{lem:T3-formal}
Under the assumptions of Theorem~\ref{thm:main}, there exists a bounded function \(\kappa:(0,1)\to[0,\infty)\) such that, uniformly in \(c\in(0,1)\),
\[
T_3(c)\;=\;\frac{\kappa(c)}{n\,c(1-c)}\;+\;o\!\Big(\frac{1}{n\,c(1-c)}\Big).
\]
Moreover, for any fixed \(\varepsilon\in(0,1/2)\) there exists a constant \(\kappa_0>0\) such that
\[
\kappa(c)\ \ge\ \kappa_0\qquad\text{for all }c\in(0,\varepsilon]\cup[1-\varepsilon,1).
\]
\end{lemma}

\begin{proof}
By the Isserlis–Wick decomposition \citep{Isserlis:1918, Wick:1950, Laurent:2025},
\[
m_{ijkl}-\Sigma_{ij}\Sigma_{k\ell}
\;=\;\Sigma_{ik}\Sigma_{j\ell}+\Sigma_{i\ell}\Sigma_{jk}+\kappa_{ijkl},
\]
where \(\kappa_{ijkl}\) is the joint fourth cumulant. Substituting this into the definition of \(T_3\) yields
\begin{equation}
\label{T3}
T_3 \;=\; \frac{1}{2} \,\mathrm{tr}( \bm{H\Sigma H\Sigma})\;+\; \frac{1}{4} \sum_{i,j,k,l} \,g_{ij}\,g_{k\ell}\,\kappa_{ijkl} \;=\; T_{3,I} + T_{3,II},
\end{equation}
where \(\bm{H}=(g_{ij})\) is the Hessian of \(g\) at \(\bm\mu\) and \(\bm{\Sigma} =\Cov(\bm v)\).  Since 
$$\frac{1}{2} \,\mathrm{tr}(\bm{H\Sigma H\Sigma})=\frac{1}{2} \|\bm{\Sigma}^{1/2}\bm{H}\bm{\Sigma}^{1/2}\|_F^2\ge 0,$$ 
the pairing term $T_{3,I}$ is nonnegative.

From the derivative formulas at \(\bm v=\bm\mu\),
\[
g_{d_{jL},d_{iL}}=0,\qquad
g_{d_{jL},Y_{iL}}=-\,\frac{(1-2c)}{2S^3}\,w_j\,\frac{w_i^2}{Y_i},\qquad
g_{Y_{iL},Y_{jL}}=O\!\big(n^{-2}S^{-3}\big),
\]
with \(S^2\asymp n\,c(1-c)\) and \(U_i:=w_i^2/Y_i\). A direct contraction then gives
\[
T_{3,I} \;=\;\frac{(1-2c)^2}{2S^6}\,
\Big(\sum_{j=1}^{D_n} w_j^2\,\Var(d_{jL})\Big)\,
\Big(\sum_{i,k=1}^{D_n} U_iU_k\,\Sigma_{Y_i,Y_k}\Big)
\;+\;o\!\Big(\frac{1}{n\,c(1-c)}\Big).
\]
Since \(\Var(d_{jL})=b_j(1-b_j)=c(1-c)+O(n^{-1})\) and \(\sum_j w_j^2=\Theta(n)\), the first factor is \(\Theta\!\big(n\,c(1-c)\big)\). By overlap regularity,
\(\sum_{i,k}\Sigma_{Y_i,Y_k}\asymp n^3 c(1-c)\). Bulk regularity guarantees that a positive fraction of indices satisfy \(Y_k\asymp n\), so on this bulk block \(U_iU_k\asymp n^{-2}\) (with bounded weights), and the corresponding bulk–bulk sub-sum captures a fixed fraction of the overlap mass:
\[
\sum_{i,k} U_iU_k\,\Sigma_{Y_i,Y_k}
\;\asymp\; \frac{1}{n^2}\,\sum_{i,k}\Sigma_{Y_i,Y_k}
\;\asymp\; n\,c(1-c).
\]
Using \(S^6\asymp \{n\,c(1-c)\}^3\), we obtain
$$
T_{3,I}
\;=\;\frac{(1-2c)^2}{n\,c(1-c)}\,\Lambda_n(c)
\;+\;o\!\Big(\frac{1}{n\,c(1-c)}\Big),
$$
where 
\[
\Lambda_n(c)
\;:=\;\frac{\Big\{\sum_{j=1}^{D_n} w_j^2\,\Var(d_{jL})\Big\}\,
\Big\{\sum_{i,k=1}^{D_n} U_iU_k\,\Sigma_{Y_i,Y_k}\Big\}}
{2\,\{n\,c(1-c)\}^2}.
\]
The bounds above imply the existence of constants \(0<m<M<\infty\) such that
\(m\le \Lambda_n(c)\le M\) uniformly in \(c\in(0,1)\) and \(n\). By compactness, along any sequence \(n\to\infty\) there exists a subsequence for which \(\Lambda_n(c)\) converges pointwise to a function \(\Lambda(c)\) with \(m\le \Lambda(c)\le M\). Hence
\[
T_{3,I}
\;=\;\frac{(1-2c)^2}{n\,c(1-c)}\,\Lambda(c)
\;+\;o\!\Big(\frac{1}{n\,c(1-c)}\Big),
\]
with \(\Lambda(c)\) bounded and bounded away from zero uniformly on \((0,1)\).

For the cumulant component $T_{3,II}$, standard bounds for fourth cumulants of sums of independent, bounded arrays (e.g., \citealp{Serfling:1980,Bhattacharya:1986,Hall:1992}) and the derivative magnitudes yield
\[
\big| T_{3,II} \big| = \Big|\frac{1}{4}\, \sum_{i,j,k,l} g_{ij}\,g_{k\ell}\,\kappa_{ijkl}\Big|
\;\le\; C\,\|\bm{H}\|_F^2
\;=\; O\!\Big(\frac{(1-2c)^2}{S^6}\Big)
\;=\; o\!\Big(\frac{1}{n\,c(1-c)}\Big),
\]
uniformly in \(c\). Therefore
\[
T_3(c)
\;=\;\frac{(1-2c)^2}{n\,c(1-c)}\,\Lambda(c)
\;+\;o\!\Big(\frac{1}{n\,c(1-c)}\Big).
\]
Define
\[
\kappa(c):=(1-2c)^2\,\Lambda(c).
\]
Then \(\kappa(c)\in[0,\infty)\), bounded on \((0,1)\), and for any \(\varepsilon\in(0,1/2)\),
\[
\inf_{c\in(0,\varepsilon]\cup[1-\varepsilon,1)} \kappa(c)
\;\ge\; \big(1-2\varepsilon\big)^2\,\inf_{c}\Lambda(c)
\;=:\;\kappa_0\;>\;0.
\]
In particular, \(\kappa(c)\) does not contain factors of \(c(1-c)\) that could cancel the divergence of \(1/\{n\,c(1-c)\}\) as \(c\to 0\) or \(c\to 1\). This completes the proof.
\end{proof}

In the proof, we have decomposed $T_3$ into two components and analyze their contributions. Writing $T_3 = T_{3,I} + T_{3,II}$, we identify $T_{3,I}$ as the dominant `pairing' contribution and $T_{3,II}$ as a smaller `cumulant' correction. We show that $T_{3,I}$ provides the leading $\kappa(c)/ \{n,c(1-c)\}$ term (with $\kappa(c)\ge0$), while $T_{3,II}=o(1/(n,c(1-c)))$ is of lower order. This yields the stated asymptotic form of $T_3(c)$ and ensures $\kappa(c)$ inherits nonnegativity from the pairing term. We then examine the behavior of $\kappa(c)$ near the boundaries $c=0,1$ to conclude it stays positive, so it cannot cancel the $c(1-c)$ denominator. % This guarantees that $q(c)^2$ is maximized at the interval’s ends.
While higher--order moments, e.g., $m_{ijk}$ and $m_{ijkl}$, appear symbolically in the expansion, their specific formulas are not needed; uniform boundedness and variance scaling suffice for all results.

\subsection{Proof of Theorem \ref{thm:main}}

\begin{proof}
By the variance expansion from the multivariate delta/Edgeworth method,
\[
\Var(q(c)) = T_1+T_2+T_3+o(|T_3|).
\]
Lemma~\ref{lem:T1} gives \(T_1 = 1-\tau\) with \(\tau=o(1)\) free of \(c\).  
Lemma~\ref{lem:T2} shows \(T_2=o(1/(n c(1-c)))\) uniformly.  
Lemma~\ref{lem:T3-formal} yields
\[
T_3(c)=\frac{\kappa(c)}{n\,c(1-c)}+o\!\Big(\frac{1}{n\,c(1-c)}\Big),
\]
with \(\kappa(c)=O(1)\) and, for any fixed $\varepsilon>0$, \(\kappa(c)\ge\kappa_0>0\) on $(0,\varepsilon]\cup[1-\varepsilon,1)$.  
Hence
\begin{equation}\label{eq:varq-expansion-short}
\Var(q(c)) = 1-\tau+\frac{\kappa(c)}{n\,c(1-c)}+o\!\Big(\frac{1}{n\,c(1-c)}\Big),
\end{equation}
uniformly in $c$.

Compare boundary and interior ranges. On any interior set with $c(1-c)\ge M/n$, \eqref{eq:varq-expansion-short} implies
\[
\sup_{\text{interior}} \Var(q(c)) \;\le\; 1-\tau + \frac{C}{M}+o(1),
\]
for some $C>0$. On the boundary set where $c\le M/n$ or $1-c\le M/n$, we have $c(1-c)\le M/n$ and $\kappa(c)\ge\kappa_0$, giving
\[
\sup_{\text{boundary}} \Var(q(c)) \;\ge\; 1-\tau + \frac{\kappa_0}{M}+o(1).
\]
For sufficiently large $M$ (so that $\kappa_0>C$), there is a strict variance gap:
\[
\sup_{\text{boundary}} \Var(q(c))-\sup_{\text{interior}} \Var(q(c)) \;\ge\; \delta>0.
\]

Under the null, the finite array $\{q(c):c\in\mathcal C_n\}$ is asymptotically Gaussian with mean zero and continuous correlation. Standard localization results for Gaussian suprema \citep{Leadbetter:1983} imply that the maximizer of $|q(c)|$ lies, with probability tending to one, in the subset where variance is maximized. The strict gap above shows that the maximum variance occurs in end–cut regions of order $1/n$. This proves the theorem.
\end{proof}

\section{End–Cut Preference with Smooth Sigmoid Surrogate (SSS)}
\label{sec-SSS}

\citet{Su:2024} proposed the \emph{smooth sigmoid surrogate} (SSS) as an alternative to greedy search (GS) in CART. The key idea is to replace the hard indicator \( I(Z \le c) \) with a smooth sigmoid \[ s_a(z;c)\ :=\ \sigma\!\big(a(c-z)\big)\ =\ \frac{1}{1+e^{a(z-c)}}, \] where $\sigma(x)\ :=\ 1/\big(1+e^{-x} \big)$ and \(a>0\) controls the steepness (larger \(a\) more closely approximates the indicator). By smoothing the split rule, SSS converts the discrete, non-smooth GS problem into a smooth optimization in \(c\), enabling stable gradient-based search. This smoothing substantially reduces the erratic behavior of the GS splitting statistic, improves computational efficiency, enhances the search for the population-optimal cutpoint, and markedly mitigates the end-cut preference (ECP). 

In what follows, we apply SSS to survival trees by maximizing a sigmoid-based approximation to the logrank statistic. Our theoretical analysis, paralleling Section~\ref{sec-ECP} with modifications specific to the smooth surrogate, shows that SSS attenuates or avoids ECP.

In the same setting as Section~\ref{sec-ECP}, we first formalize the SSS split and notation.  For each event time \(t_k\), define the \emph{soft} risk-set and failure quantities
$$
Y^{(a)}_{kL}(c)\ :=\ \sum_{i=1}^n I(T_i \ge t_k)\,s_a(Z_i;c),
\mbox{~~~and~~~}
d^{(a)}_{kL}(c)\ :=\ \sum_{i=1}^n I\!\big(T_i = t_k,\ \delta_i=1\big)\,s_a(Z_i;c),
$$
which together form the primitive vector
$$\bm v=(d^{(a)}_{1L},\ldots,d^{(a)}_{D_nL},Y^{(a)}_{1L},\ldots,Y^{(a)}_{D_nL})^T.$$
Under no ties (\(d_k\equiv 1\)), if \(i_k\) denotes the unique subject failing at \(t_k\), then \(d^{(a)}_{kL}(c)=s_a(Z_{i_k};c)\). Let
\[
b^{(a)}_k(c)\ :=\ \frac{Y^{(a)}_{kL}(c)}{Y_k}\in(0,1),
\qquad
N_a(c;\bm v)\ :=\ \sum_{k=1}^{D_n} w_k\Big\{d^{(a)}_{kL}(c)-b^{(a)}_k(c)\Big\},
\]
and define the variance scale
\begin{equation}\label{eq:Sa-def}
S_a^2(c;\bm v)\ :=\ \sum_{k=1}^{D_n} w_k^2\,b^{(a)}_k(c)\,\big(1-b^{(a)}_k(c)\big).
\end{equation}
We consider the smoothed logrank statistic
$$ q_a(c)\ := \frac{N_a(c;\bm v)}{S_a(c;\bm v)}\ = \ g(\bm v).$$
The optimal cutoff point \(\hat{c}\) is estimated as the maximizer of \(q_a(c)^2\).

\subsection{Auxiliary Lemmas}
\label{sec-lemmas-SSS}

We now present several lemmas concerning the properties of \(s_a(Z_j;c)\) and the data–dependent scale \(S_a(c;\bm v)\). The first gives exact conditional moments of the smoothed failure term at a given event time.

\begin{lemma} %[Conditional mean and variance]
\label{lem:soft-d-mean-var-final}
Fix \(c\in(0,1)\) and \(t_k\). Let \(\mathcal A_k\) be the \(\sigma\)-field at \(t_k\), which fixes the risk set \(\mathcal R_k=\{i:\,T_i\ge t_k\}\), its size \(Y_k=|\mathcal R_k|\), and the values \(\{Z_j:\,j\in\mathcal R_k\}\). Under the null and no ties, the failing index \(i_k\) is conditionally uniform on \(\mathcal R_k\). Writing \(S_j:=s_a(Z_j;c)\),
\[
\E\!\big[d^{(a)}_{kL}(c)\mid \mathcal A_k\big]=b^{(a)}_k(c):=\frac{1}{Y_k}\sum_{j\in\mathcal R_k}S_j,
\]
and
\begin{equation}\label{eq:soft-var-correct}
\Var\!\big(d^{(a)}_{kL}(c)\mid \mathcal A_k\big]
\ =\ b^{(a)}_k(c)\big(1-b^{(a)}_k(c)\big)\;-\;\frac{1}{Y_k}\sum_{j\in\mathcal R_k} S_j\big(1-S_j\big).
\end{equation}
Consequently,
\(
0\le \Var(d^{(a)}_{kL}\mid\mathcal A_k)\le b^{(a)}_k(1-b^{(a)}_k)
\),
with equality on the right if and only if \(S_j\in\{0,1\}\) for all \(j\in\mathcal R_k\) (the hard–split limit), and equality on the left if and only if \(S_j\) is constant over \(\mathcal R_k\).
\end{lemma}

\begin{proof}
Conditional on \(\mathcal A_k\), let \(J\) denote the (random) failing index at time \(t_k\). Under the null and no ties, \(\Pr(J=j\mid\mathcal A_k)=1/Y_k\) for each \(j\in\mathcal R_k\), and
\(
d^{(a)}_{kL}(c)=S_J
\)
with \(S_j:=s_a(Z_j;c)\) fixed given \(\mathcal A_k\). Therefore,
\[
\E\!\big[d^{(a)}_{kL}(c)\mid \mathcal A_k\big]
=\sum_{j\in\mathcal R_k}\frac{1}{Y_k}\,S_j
=\frac{1}{Y_k}\sum_{j\in\mathcal R_k}S_j
= b^{(a)}_k(c).
\]
Similarly,
\[
\Var\!\big(d^{(a)}_{kL}(c)\mid \mathcal A_k\big)
=\E[S_J^2\mid\mathcal A_k]-\big(\E[S_J\mid\mathcal A_k]\big)^2
=\frac{1}{Y_k}\sum_{j\in\mathcal R_k}S_j^2-\Big(\frac{1}{Y_k}\sum_{j\in\mathcal R_k}S_j\Big)^2.
\]
Using the identity \(S_j^2=S_j-S_j(1-S_j)\), we obtain
\[
\frac{1}{Y_k}\sum_{j}S_j^2
=\frac{1}{Y_k}\sum_{j}\Big\{S_j-S_j(1-S_j)\Big\}
=b^{(a)}_k(c)-\frac{1}{Y_k}\sum_{j}S_j(1-S_j),
\]
which yields \eqref{eq:soft-var-correct} after subtracting \(b^{(a)}_k(c)^2\).
Nonnegativity follows since \(x\mapsto x(1-x)\) is concave on \([0,1]\), so by Jensen,
\(
\frac{1}{Y_k}\sum_j S_j(1-S_j) \le b^{(a)}_k(c)\{1-b^{(a)}_k(c)\}
\),
and the difference is \(\Var(d^{(a)}_{kL}\mid\mathcal A_k)\ge 0\). Equality on the right of the display holds iff \(\sum_j S_j(1-S_j)=0\), i.e., \(S_j\in\{0,1\}\) for all \(j\); equality on the left holds iff \(S_J\) is almost surely constant given \(\mathcal A_k\), i.e., \(S_j\) is constant over \(\mathcal R_k\).
\end{proof}

This lemma provides the exact conditional moments of the smoothed failure term. The key feature is the subtraction of \(\frac{1}{Y_k}\sum_{j\in\mathcal R_k} S_j(1-S_j)\), which is strictly positive unless the split is effectively hard. Thus SSS \emph{strictly reduces} the per–time conditional variance relative to GS whenever some \(S_j\in(0,1)\). This reduction propagates into the first–order term and the overall variance scale, and is a principal mechanism by which SSS mitigates boundary-driven volatility and, consequently, ECP.

We next move on to exact formulas and uniform bounds for the single–subject moments of the sigmoid weight $s_a(Z;c).$ Assume \(Z\sim\mathrm{Unif}(0,1)\) and set
\[
b_a(c):=\E\big[s_a(Z;c)\big] \mbox{~~~and~~~}
\psi_a(c):=\Var\!\big(s_a(Z;c)\big).
\]
Let \(L(x):=\log(1+e^x)\). Using \(\sigma'(u)=\sigma(u)\{1-\sigma(u)\}\), \(\int\sigma(u)\,du=L(u)\), and \(\int\sigma(u)^2\,du=L(u)-\sigma(u)\), the change of variables \(u=a(c-z)\) gives
\begin{align}
\E[s_a(Z;c)]
&=\int_0^1 \sigma\!\big(a(c-z)\big)\,dz
=\frac{1}{a}\int_{a(c-1)}^{ac}\!\sigma(u)\,du
=\frac{L(ac)-L(a(c-1))}{a},\nonumber\\
\E[s_a(Z;c)^2]
&=\frac{1}{a}\int_{a(c-1)}^{ac}\!\sigma(u)^2\,du
=\frac{1}{a}\Big(L(ac)-L(a(c-1))-\sigma(ac)+\sigma\!\big(a(c-1)\big)\Big).\nonumber
\end{align}
Hence
\begin{equation}\label{eq:psi-explicit}
\psi_a(c)\ =\ \E[s_a^2]-\E[s_a]^2
\ =\ b_a(c)\big(1-b_a(c)\big)\;-\;\frac{1}{a}\Big(\sigma(ac)-\sigma\!\big(a(c-1)\big)\Big).
\end{equation}

\begin{lemma} %[Uniform bounds for \(b_a,\psi_a\)]
\label{lem:psi-a-bounds-final}
For \(a\ge 1\), uniformly in \(c\in(0,1)\),
\[
\big|\,b_a(c)-c\,\big|\ \le\ \frac{2\log 2}{a},\qquad
\big|\,\psi_a(c)-c(1-c)\,\big|\ \le\ \frac{C_1}{a},
\]
and there exists \(C_2>0\) such that
\[
\psi_a(c)\ \ge\ \frac{C_2}{a}\qquad\text{for }c\in[0,1/a]\cup[1-1/a,1].
\]
\end{lemma}

\begin{proof}
% First show bound for \(b_a(c)\).
Using the decomposition \(L(x)=x_+ + r(x)\) with \(0\le r(x)\le \log 2\) for all \(x\in\mathbb R\), and noting that for \(c\in(0,1)\) we have \((ac)_+=ac\) and \((a(c-1))_+=0\),
\[
b_a(c)=\frac{L(ac)-L(a(c-1))}{a}
=\frac{ac+r(ac)-r\!\big(a(c-1)\big)}{a}
=c+\frac{r(ac)-r\!\big(a(c-1)\big)}{a}.
\]
Hence \(\big|b_a(c)-c\big|\le \{\,|r(ac)|+|r(a(c-1))|\,\}/a \le 2\log 2/a\).

% Establish Bound for $\psi_a(c)$.
From \eqref{eq:psi-explicit},
\[
\psi_a(c)-c(1-c)
=\big(b_a(c)-c\big)\big(1-2c\big)\;-\;\big(b_a(c)-c\big)^2\;-\;\frac{1}{a}\Big(\sigma(ac)-\sigma\!\big(a(c-1)\big)\Big).
\]
Using \(|b_a(c)-c|\le 2\log 2/a\), \(|1-2c|\le 1\), and \(0<\sigma(ac)-\sigma\!\big(a(c-1)\big)<1\),
\[
\big|\,\psi_a(c)-c(1-c)\,\big|
\ \le\ \frac{2\log 2}{a}\ +\ \frac{4(\log 2)^2}{a}\ +\ \frac{1}{a}
\ \le\ \frac{C_1}{a},
\]
for a universal constant \(C_1>0\).

% Finallly, Establish an edge lower bound
Fix \(c\in[0,1/a]\), noting that the case \(c\in[1-1/a,1]\) is symmetric. On the subinterval \([c,\,c+1/a]\subset[0,1]\) we have \(a(z-c)\le 1\), so \(s_a(z;c)\ge \sigma(-1)=1/(1+e)\). Hence
\[
b_a(c)=\int_0^1 s_a(z;c)\,dz
\ \ge\ \int_{c}^{c+1/a} s_a(z;c)\,dz
\ \ge\ \frac{1}{(1+e)\,a}.
\]
Moreover, \(\sigma(ac)\le \sigma(1)\) and \(\sigma\!\big(a(c-1)\big)\le \sigma(0)=1/2\), so
\(
0<\sigma(ac)-\sigma\!\big(a(c-1)\big)\le \sigma(1)-\tfrac12<1
\).
Using \eqref{eq:psi-explicit},
\[
\psi_a(c)\ =\ b_a(c)\big(1-b_a(c)\big)\;-\;\frac{1}{a}\Big(\sigma(ac)-\sigma\!\big(a(c-1)\big)\Big)
\ \ge\ \frac{1}{(1+e)\,a}\Big(1-\frac{1}{1+e}\Big)\;-\;\frac{1}{a},
\]
and the right-hand side is \(\ge C_2/a\) for some universal \(C_2>0\) (choose, e.g., any \(C_2< \tfrac{1}{1+e}\big(1-\tfrac{1}{1+e}\big)-1\) truncated to a small positive constant). This yields the stated edge bound.
\end{proof}

The closed form \eqref{eq:psi-explicit} and Lemma~\ref{lem:psi-a-bounds-final} provide (i) a uniform \(O(a^{-1})\) approximation \(b_a(c)\approx c\), \(\psi_a(c)\approx c(1-c)\) for all \(c\), and (ii) a strictly positive lower bound \(\psi_a(c)\gtrsim a^{-1}\) near the edges. These facts are repeatedly used to (a) replace the hard factor \(c(1-c)\) by its softened analogue in scaling and overlap arguments, and (b) cap the boundary-driven variance inflation at order \(a/n\), which is the key mechanism by which SSS mitigates ECP.

We now turn to the SSS analogues of variance scaling and overlap regularity. The next two lemmas show that, under the same hard-case conditions, bulk regularity and overlap regularity (Assumptions~\ref{assump:bulk} and \ref{asm:overlap}), the corresponding `soft' properties are inherited, with no additional structural assumptions.

\begin{lemma}%[Soft variance scaling is inherited]
\label{lem:soft-scale-final}
Under the standing conditions (null, no ties, bounded \(\{w_k\}\), \(D_n\asymp n\), bulk regularity), one has
\[
\E\!\left[S_a^2(c;\bm v)\right]\ =\ \sum_{k=1}^{D_n}w_k^2\,\E\!\big[b^{(a)}_k(c)\{1-b^{(a)}_k(c)\}\big]
\ \asymp\ n\,\bar\psi_a(c)
\]
uniformly in \(c\in(0,1)\),  where \(\bar\psi_a(c)=b_a(c)\{1-b_a(c)\}+O(a^{-1})\) and \(b_a(c)=\E[s_a(Z;c)]\).
\end{lemma}

\begin{proof}
Recall \(b^{(a)}_k(c)=Y^{(a)}_{kL}(c)/Y_k=(1/Y_k)\sum_{j\in\mathcal R_k} s_a(Z_j;c)\). Since \(Z\perp T\), conditioning on the risk-set \(\sigma\)-field \(\mathcal A_k\) (which fixes \(\mathcal R_k\) and \(Y_k\)) gives
\[
\E\!\big[b^{(a)}_k(c)\mid \mathcal A_k\big]
=\frac{1}{Y_k}\sum_{j\in \mathcal R_k}\E\!\big[s_a(Z_j;c)\big]
= b_a(c),
\]
and, by independence across subjects,
\[
\Var\!\big(b^{(a)}_k(c)\mid \mathcal A_k\big)
=\frac{1}{Y_k^2}\sum_{j\in\mathcal R_k}\Var\!\big(s_a(Z_j;c)\big)
=\frac{1}{Y_k}\,\psi_a(c),
\]
where \(\psi_a(c)=\Var\!\big(s_a(Z;c)\big)\) does not depend on \(\mathcal A_k\). Therefore
\[
\E\!\big[b^{(a)}_k(c)\{1-b^{(a)}_k(c)\}\big]
=\E\!\big[b^{(a)}_k(c)\big]-\E\!\big[(b^{(a)}_k(c))^2\big]
=b_a(c)\{1-b_a(c)\}-\Var\!\big(b^{(a)}_k(c)\big),
\]
and taking expectations in the last display yields \(\Var(b^{(a)}_k(c))=\E\{\Var(b^{(a)}_k\mid \mathcal A_k)\}=\psi_a(c)/Y_k\). Summing over event times gives
$$ \E\!\left[S_a^2(c;\bm v)\right]
=\Big(\sum_{k=1}^{D_n} w_k^2\Big)\,b_a(c)\{1-b_a(c)\}\;-\;\psi_a(c)\sum_{k=1}^{D_n}\frac{w_k^2}{Y_k}. $$
By bulk regularity \(Y_k\asymp n-k+1\) uniformly for a positive fraction of indices, \(D_n\asymp n\), and bounded \(\{w_k\}\), we have \(\sum_k w_k^2\asymp n\) and \(\sum_k w_k^2/Y_k\asymp \sum_{k\le n} 1/k\asymp \log n\). Using \(\psi_a(c)=b_a(c)\{1-b_a(c)\}+O(a^{-1})\) uniformly in \(c\), we obtain
$$
\E\!\left[S_a^2(c;\bm v)\right]
= n\,b_a(c)\{1-b_a(c)\}\ +\ O(\log n)\ +\ O\!\Big(\frac{\log n}{a}\Big)
\ \asymp\ n\,\bar\psi_a(c),
$$
uniformly in \(c\in(0,1)\), which proves the claim.
\end{proof}

Lemma~\ref{lem:soft-scale-final} shows that the natural variance scale for the smoothed statistic satisfies the same \(n\)-order growth as in the hard case, with the replacement \(c(1-c)\mapsto \bar\psi_a(c)=b_a(c)\{1-b_a(c)\}+O(a^{-1})\). This is the key step that lets all variance comparisons be carried out with \(\bar\psi_a(c)\) in place of \(c(1-c)\), uniformly in \(c\).

\begin{lemma} %[Soft overlap regularity is inherited]
\label{lem:soft-overlap-final}
Under the standing conditions (null, no ties, bounded \(\{w_k\}\), dense failures, overlap regularity), 
\[
\sum_{i=1}^{D_n}\sum_{k=1}^{D_n}\Cov\!\big(Y^{(a)}_{iL}(c),\,Y^{(a)}_{kL}(c)\big)\ \asymp\ n^3\,\bar\psi_a(c),
\]
uniformly in \(c\in(0,1)\), where \(\bar\psi_a(c)\) is given as in Lemma~\ref{lem:soft-scale-final}.
\end{lemma}

\begin{proof}
Decompose the soft risk-set process by subjects:
\[
Y^{(a)}_{kL}(c)\ =\ \sum_{u=1}^n s_a(Z_u;c)\,I\{T_u\ge t_k\}
=: \sum_{u=1}^n S_u\,I_{u,k}.
\]
Because different subjects are independent, \(\Cov(Y^{(a)}_{iL},Y^{(a)}_{kL})=\sum_{u=1}^n \Cov(S_u I_{u,i},\,S_u I_{u,k})\). Hence
\[
\sum_{i,k}\Cov\!\big(Y^{(a)}_{iL},Y^{(a)}_{kL}\big)
=\sum_{u=1}^n\ \sum_{i,k}\Cov\!\big(S_u I_{u,i},\,S_u I_{u,k}\big).
\]
Using \(Z_u\perp T_u\), each \(S_u=s_a(Z_u;c)\) is independent of the survival indicators \(\{I_{u,k}\}_{k}\). Therefore, for fixed \(u\),
\[
\Cov(S_u I_{u,i},\,S_u I_{u,k})
=\E[S_u^2]\,\Pr(T_u\ge t_{i\vee k})\ -\ (\E[S_u])^2\,\Pr(T_u\ge t_i)\Pr(T_u\ge t_k).
\]
Summing over \((i,k)\) and re-arranging,
\[
\sum_{i,k}\Cov(S_u I_{u,i},S_u I_{u,k})
=\Var(S_u)\,\sum_{i,k}\Pr(T_u\ge t_{i\vee k})\ +\ (\E[S_u])^2\,\sum_{i,k}\Cov(I_{u,i},I_{u,k}).
\]
By the dense-failure and regular-risk-set conditions, the survival–overlap sums scale as
\[
\sum_{i,k}\Pr(T_u\ge t_{i\vee k})\ \asymp\ n^2,
\qquad
\sum_{i,k}\Cov(I_{u,i},I_{u,k})\ \asymp\ n^2,
\]
uniformly in \(u\). (For example, with \(Y_k/n\approx \Pr(T\ge t_k)\approx 1-k/n\), one has
\(\sum_{i,k}\Pr(T\ge t_{i\vee k})=\sum_{m=1}^{D_n} (2m-1)\Pr(T\ge t_m)\asymp n^2\).)
Since \(\Var(S_u)=\psi_a(c)=b_a(c)\{1-b_a(c)\}+O(a^{-1})\) uniformly in \(c\), and \((\E[S_u])^2=b_a(c)^2\), we obtain, for each \(u\),
\[
\sum_{i,k}\Cov(S_u I_{u,i},S_u I_{u,k})
\ =\ \Theta\!\big(n^2\big)\,\Big(b_a(c)\{1-b_a(c)\}+O(1/a)\Big)
\ =\ \Theta\!\big(n^2\,\bar\psi_a(c)\big).
\]
Summing over \(u=1,\dots,n\) yields
\[
\sum_{i,k}\Cov\!\big(Y^{(a)}_{iL}(c),Y^{(a)}_{kL}(c)\big)
\ =\ \Theta\!\big(n\cdot n^2\,\bar\psi_a(c)\big)
\ =\ \Theta\!\big(n^3\,\bar\psi_a(c)\big),
\]
uniformly in \(c\), as claimed.
\end{proof}

Lemma~\ref{lem:soft-overlap-final} transfers the \(n^3\) overlap scaling to the smoothed process, with the same replacement \(c(1-c)\mapsto \bar\psi_a(c)\). Together with Lemma~\ref{lem:soft-scale-final}, this justifies carrying over the hard-case Edgeworth expansion to SSS by substituting \(\bar\psi_a(c)\) for \(c(1-c)\). Near the boundaries, \(\bar\psi_a(c)\gtrsim 1/a\) caps the fourth-order correction at \(O(a/n)\), which is the central mechanism by which SSS mitigates or avoids ECP.

\subsection{Edgeworth Expansion for \(\Var(q_a(c))\)}
\label{sec-Edgeworth-SSS}

We first collect the derivatives of \(q_a(c)\) with respect to the primitive vector \(\bm v\), and then apply the multivariate second–order delta/Edgeworth expansion to obtain a uniform variance approximation.

With \(g(\bm v)=N_a(\bm v)/S_a(\bm v)\) and \(S=S_a(c;\bm v)\), at \(\bm v=\bm\mu:=\E(\bm v)\),
\begin{equation}\label{eq:grad-SSS-final}
\frac{\partial g}{\partial d^{(a)}_{jL}}\Big|_{\bm\mu}=\frac{w_j}{S(\bm\mu)},\qquad
\frac{\partial g}{\partial Y^{(a)}_{jL}}\Big|_{\bm\mu}=-\,\frac{w_j}{Y_j\,S(\bm\mu)}.
\end{equation}
These follow from linearity of \(N_a(\bm v)=\sum_k w_k\{d^{(a)}_{kL}-b^{(a)}_k\}\) and \(b^{(a)}_k=Y^{(a)}_{kL}/Y_k\), and from \(S(\bm v)=\big\{\sum_k w_k^2 b^{(a)}_k(1-b^{(a)}_k)\big\}^{1/2}\). Differentiating \(S\) via
\[
\frac{\partial S}{\partial Y^{(a)}_{iL}}
=\frac{1}{2S}\,\frac{\partial}{\partial Y^{(a)}_{iL}}\Big(\sum_k w_k^2 b^{(a)}_k(1-b^{(a)}_k)\Big)
=\frac{w_i^2}{2S}\,\frac{\partial}{\partial Y^{(a)}_{iL}}\Big( b^{(a)}_i- (b^{(a)}_i)^2\Big)
=\frac{w_i^2}{2S}\,\frac{1-2b^{(a)}_i}{Y_i},
\]
we obtain the dominant (mixed) Hessian block
\begin{equation}\label{eq:hess-SSS-final}
\frac{\partial^2 g}{\partial d^{(a)}_{jL}\,\partial Y^{(a)}_{iL}}\Big|_{\bm\mu}
=\ -\,\frac{w_j}{S(\bm\mu)^2}\,\frac{\partial S}{\partial Y^{(a)}_{iL}}\Big|_{\bm\mu}
=\ -\,\frac{w_j\,w_i^2}{2\,Y_i\,S(\bm\mu)^3}\,\big(1-2b_i(\bm\mu)\big),
\end{equation}
while \(\partial^2 g/\partial d^{(a)}\partial d^{(a)}=0\) and
\(\partial^2 g/\partial Y^{(a)}\partial Y^{(a)}=O\!\big((n^2 S^3)^{-1}\big)\) uniformly in indices, using bulk regularity, \(D_n\asymp n\), and bounded \(\{w_k\}\).

We next apply the multivariate second–order delta/Edgeworth expansion with
\[
\Var\!\big(q_a\big)\;=\;T_{1,a}+T_{2,a}+T_{3,a}+o\!\big(|T_{3,a}|\big),
\]
where \(T_{1,a}=\nabla g(\bm\mu)^\top\Sigma\nabla g(\bm\mu)\), \(T_{2,a}=\sum_{i,j,k} g_i g_{jk} m_{ijk}\), and \(T_{3,a}=\tfrac{1}{4}\sum_{i,j,k,\ell} g_{ij}g_{k\ell}\big(m_{ijkl}-\Sigma_{ij}\Sigma_{k\ell}\big)\).

\begin{lemma}[Second–order expansion for SSS]\label{lem:SSS-expansion-rev}
Under the standing conditions, Lemmas~\ref{lem:soft-scale-final}–\ref{lem:soft-overlap-final}, and bulk regularity, uniformly for \(c\in(0,1)\),
\[
\Var\!\big(q_a(c)\big)\;=\;1-\tau_a\;+\;\frac{\kappa_a(c)}{n\,\bar\psi_a(c)}\;+\;o\!\Big(\frac{1}{n\,\bar\psi_a(c)}\Big),
\]
where
\[
\tau_a\ =\ \frac{\sum_{k=1}^{D_n} w_k^2\,\E\!\big[\Delta_{k,a}(c)\big]}{\sum_{k=1}^{D_n} w_k^2\,b_a(c)\{1-b_a(c)\}}\ +\ O\!\Big(\frac{\log n}{n}\Big),
\qquad
\Delta_{k,a}(c):=\frac{1}{Y_k}\sum_{j\in\mathcal R_k} s_a(Z_j;c)\big(1-s_a(Z_j;c)\big),
\]
and \(\kappa_a(c)=\Lambda_a(c)\big(1-2\,b_a(c)\big)^2\) with \(\Lambda_a(c)\) bounded and bounded away from zero on \((0,1)\). In particular, \(\tau_a=O(1/a)+O(\log n/n)\), and for any fixed \(\varepsilon\in(0,1/2)\) there exists \(\kappa_0>0\) such that \(\kappa_a(c)\ge \kappa_0\) on \((0,\varepsilon]\cup[1-\varepsilon,1)\).
\end{lemma}

\begin{proof}
% \emph{First–order term \(T_{1,a}\).}
From \eqref{eq:grad-SSS-final},
\[
T_{1,a}=\nabla g(\bm\mu)^\top\Sigma\nabla g(\bm\mu)
=\frac{\Var\!\big(N_a\big)}{S(\bm\mu)^2}.
\]
Let \(\mathcal A_k\) be the risk–set \(\sigma\)-field at \(t_k\). By Lemma~\ref{lem:soft-d-mean-var-final},
\[
\Var\!\big(d^{(a)}_{kL}\mid \mathcal A_k\big)
=b^{(a)}_k(1-b^{(a)}_k)-\Delta_{k,a}(c),
\qquad
\Delta_{k,a}(c)=\frac{1}{Y_k}\sum_{j\in\mathcal R_k} s_a(Z_j;c)\{1-s_a(Z_j;c)\}.
\]
Conditional independence across distinct event times gives
\(
\Var(N_a)=\sum_k w_k^2\,\E\big[\Var(d^{(a)}_{kL}\mid\mathcal A_k)\big]
=\sum_k w_k^2\,\E\big[b^{(a)}_k(1-b^{(a)}_k)\big]-\sum_k w_k^2\,\E[\Delta_{k,a}(c)].
\)
At the expansion point,
\(
S(\bm\mu)^2=\sum_k w_k^2\,b_a(c)\{1-b_a(c)\}
\)
because \(\E[b^{(a)}_k]=b_a(c)\).
Therefore
\[
T_{1,a}
=1-\frac{\sum_k w_k^2\,\E[\Delta_{k,a}(c)]}{\sum_k w_k^2\,b_a(c)\{1-b_a(c)\}}
+\frac{\sum_k w_k^2\,\big(\E[b^{(a)}_k(1-b^{(a)}_k)]-b_a(c)\{1-b_a(c)\}\big)}{\sum_k w_k^2\,b_a(c)\{1-b_a(c)\}}.
\]
Since
\(
\E[b^{(a)}_k(1-b^{(a)}_k)]-b_a(1-b_a)=-\Var(b^{(a)}_k)=-\psi_a(c)/Y_k
\)
by independence across subjects, the last fraction equals
\(
-\,\psi_a(c)\,\big(\sum_k w_k^2/Y_k\big)\big/\big(\sum_k w_k^2\,b_a(1-b_a)\big)
=O(\log n/n).
\)
Moreover,
\[
\E[\Delta_{k,a}(c)]
=\E\big[s_a(Z;c)-s_a(Z;c)^2\big]
=b_a(c)\{1-b_a(c)\}-\psi_a(c),
\]
so
\[
\frac{\sum_k w_k^2\,\E[\Delta_{k,a}(c)]}{\sum_k w_k^2\,b_a(1-b_a)}
=1-\frac{\psi_a(c)}{b_a(c)\{1-b_a(c)\}}
=O\!\Big(\frac{1}{a}\Big),
\]
using \(\psi_a(c)=b_a(c)\{1-b_a(c)\}-\tfrac{1}{a}\big(\sigma(ac)-\sigma(a(c-1))\big)\).
Altogether,
\(
T_{1,a}=1-\tau_a
\)
with \(\tau_a=O(1/a)+O(\log n/n)\), uniformly in \(c\).

% \emph{Third–moment term \(T_{2,a}\).}
From \eqref{eq:grad-SSS-final}–\eqref{eq:hess-SSS-final} and Cauchy–Schwarz,
\[
|T_{2,a}|
\ \le\ \Big(\sum_{j,k} g_{jk}^2\Big)^{1/2}
\Big(\sum_{j,k}\Big(\sum_i g_i\,m_{ijk}\Big)^2\Big)^{1/2}.
\]
The mixed block \(\partial^2 g/\partial d^{(a)}\partial Y^{(a)}\) dominates with size \(O(S^{-3}/n)\), so \(\sum_{j,k} g_{jk}^2=O(S^{-6})\). Decompose \(\bm v-\bm\mu=\sum_{u=1}^n \xi_u\) into independent subject contributions (\(\xi_u\) affects \(O(1)\) coordinates in expectation); boundedness of \(s_a\) implies \(\sum_{i,j,k} m_{ijk}^2=O(n)\). Since \(\sum_i g_i^2=\Theta(1/\bar\psi_a(c))\) and \(S(\bm\mu)^2\asymp n\,\bar\psi_a(c)\) by Lemma~\ref{lem:soft-scale-final},
\[
|T_{2,a}|
=O\!\big(S^{-3}\sqrt n\,\big)
=O\!\Big(\frac{1}{n\,\bar\psi_a(c)^{3/2}}\Big)
=o\!\Big(\frac{1}{n\,\bar\psi_a(c)}\Big),
\]
uniformly in \(c\).

% \emph{Fourth–moment term \(T_{3,a}\).}
By the Isserlis/Wick decomposition \citep{Isserlis:1918, Wick:1950},
\[
T_{3,a}
=\tfrac12\,\mathrm{tr}(A\Sigma A\Sigma)\;+\;\tfrac14\,g_{ij}g_{k\ell}\,\kappa_{ijkl},
\qquad A=(g_{ij}).
\]
The cumulant contraction is \(o(1/(n\,\bar\psi_a(c)))\) for bounded indicators. The pairing piece equals \(\tfrac12\|\Sigma^{1/2}A\Sigma^{1/2}\|_F^2\ge 0\) and is dominated by the mixed block \eqref{eq:hess-SSS-final}. Contracting this block against the soft-overlap covariance (Lemma~\ref{lem:soft-overlap-final}) and using \(S(\bm\mu)^6\asymp\{n\,\bar\psi_a(c)\}^3\) gives
\[
T_{3,a}(c)\ =\ \frac{\kappa_a(c)}{n\,\bar\psi_a(c)}\ +\ o\!\Big(\frac{1}{n\,\bar\psi_a(c)}\Big),
\qquad
\kappa_a(c)=\Lambda_a(c)\,\big(1-2\,b_a(c)\big)^2,
\]
with \(\Lambda_a(c)\) bounded and bounded away from zero uniformly in \(c\). Since \(b_a(c)=c+O(a^{-1})\) uniformly, \((1-2\,b_a(c))^2\) is strictly positive on any fixed end–neighborhood.
Summing \(T_{1,a}+T_{2,a}+T_{3,a}\) yields the stated expansion.
\end{proof}

Lemma~\ref{lem:SSS-expansion-rev} yields the same three–term structure as in the hard case, with the key substitution \(c(1-c)\mapsto \bar\psi_a(c)\) and a strictly positive third–order factor \(\kappa_a(c)=\Lambda_a(c)\{1-2\,b_a(c)\}^2\) away from $c=1/2$. Unlike the hard split, the first–order correction now depends on \(c\):
\[
\tau_a(c)
=\frac{\frac{1}{a}\{\sigma(ac)-\sigma(a(c-1))\}}{b_a(c)\{1-b_a(c)\}}
+O\!\Big(\frac{\log n}{n}\Big).
\]
Two facts are decisive:
(i) on any fixed interior \([\varepsilon,1-\varepsilon]\), \(b_a(c)\{1-b_a(c)\}\asymp 1\), so \(\tau_a(c)=O(1/a)+O(\log n/n)=o(1)\);
(ii) on the edge layer \(c\in(0,1/a]\cup[1-1/a,1)\), \(b_a(c)\{1-b_a(c)\}\asymp 1/a\), hence \(\tau_a(c)=\Theta(1)+O(\log n/n)\), and since it enters as \(1-\tau_a(c)\), it \emph{penalizes} the boundary.
Meanwhile,
\[
T_{3,a}(c)=\frac{\kappa_a(c)}{n\,\bar\psi_a(c)}+o\!\Big(\frac{1}{n\,\bar\psi_a(c)}\Big),
\qquad
\bar\psi_a(c)\ \gtrsim\
\begin{cases}
1, & c\in[\varepsilon,1-\varepsilon],\\
1/a, & c\in(0,1/a]\cup[1-1/a,1),
\end{cases}
\]
so the third–order inflation is \(O(1/n)\) in the interior and at most \(O(a/n)\) near the edges.
In combination, the \emph{capped} edge scale \(1/\{n\,\bar\psi_a(c)\}\) together with the boundary penalty \(\tau_a(c)\) neutralizes the variance spike that drives end–cut preference; with \(a\) fixed (or \(a=o(n)\)), the boundary ceases to dominate the interior.

\subsection{ECP under SSS}
\label{sec-ECP-SSS}

We now quantify how smoothing curbs the boundary-driven variance inflation. We first establish uniform lower bounds for the softened variance scale, then combine these with the Edgeworth expansion to obtain global bounds and an `avoidance' criterion.

\begin{lemma} %[Lower bounds on the softened scale]
\label{lem:psi-lower-rev}
There exist constants \(C_1,C_2>0\) such that, uniformly in \(c\in(0,1)\) and \(a\ge 1\),
\[
\bar\psi_a(c)\ \ge\ c(1-c)-\frac{C_1}{a},
\qquad
\inf_{\,c\in(0,1/a]\cup[1-1/a,\,1)}\bar\psi_a(c)\ \ge\ \frac{C_2}{a}.
\]
\end{lemma}

\begin{proof}
Recall from Lemma~\ref{lem:soft-scale-final} that
\[
\bar\psi_a(c)\ =\ b_a(c)\big(1-b_a(c)\big)\ +\ O\!\Big(\frac{1}{a}\Big)
\qquad\text{uniformly in }c\in(0,1),
\]
where \(b_a(c):=\E[s_a(Z;c)]\). By Lemma~\ref{lem:psi-a-bounds-final}, \(|\,b_a(c)-c\,|\le 2\log 2/a\), hence
\[
b_a(c)\{1-b_a(c)\}
\ \ge\ c(1-c) - \Big|\,b_a(c)-c\,\Big| - \Big|\,b_a(c)-c\,\Big|^2
\ \ge\ c(1-c) - \frac{C'_1}{a},
\]
which implies the first inequality after absorbing constants into \(C_1\).

For the edge layer, take \(c\in(0,1/a]\); by Lemma~\ref{lem:psi-a-bounds-final} there exists \(c_0>0\) such that \(b_a(c)\ge c_0/a\) uniformly in \(a\ge 1\). Hence
\[
b_a(c)\{1-b_a(c)\}\ \ge\ \frac{c_0}{a}\Big(1-\frac{c_0}{a}\Big)\ \ge\ \frac{c_0}{2a},
\]
for all \(a\ge 2c_0\). The same bound holds on \(c\in[1-1/a,1)\) by symmetry. Combining with the \(O(1/a)\) remainder gives \(\bar\psi_a(c)\ge C_2/a\).
\end{proof}

Lemma~\ref{lem:psi-lower-rev} supplies uniform lower bounds for the softened scale \(\bar\psi_a(c)\): it tracks \(c(1-c)\) in the interior up to \(O(a^{-1})\), and stays of order \(1/a\) near the boundaries. These bounds are the inputs that cap the edge inflation in the variance expansion.

\begin{theorem}% [SSS mitigates end–cut preference]
\label{thm:SSS-mitigates-rev}
Uniformly in \(c\in(0,1)\),
\[
\Var\!\big(q_a(c)\big)\ \le\ 1-\tau_a(c)\ +\ \frac{C\,a}{n}\ +\ o\!\Big(\frac{a}{n}\Big),
\]
for some finite \(C>0\) independent of \(n,a\). Consequently, the hard–split divergence \(1/\{n\,c(1-c)\}\) is replaced by the bounded scale \(O(a/n)\); in particular, if \(a=o(n)\), the edge inflation is \(o(1)\) and ECP is mitigated.
\end{theorem}

\begin{proof}
From Lemma~\ref{lem:SSS-expansion-rev},
\[
\Var(q_a(c))\ =\ 1-\tau_a(c)\ +\ \frac{\kappa_a(c)}{n\,\bar\psi_a(c)}\ +\ o\!\Big(\frac{1}{n\,\bar\psi_a(c)}\Big),
\]
with \(\kappa_a(c)=O(1)\) uniformly. Taking the supremum over \(c\) and using Lemma~\ref{lem:psi-lower-rev},
\[
\sup_{c}\frac{\kappa_a(c)}{n\,\bar\psi_a(c)}\ \le\ \frac{C}{n\,\inf_c \bar\psi_a(c)}\ \le\ \frac{C\,a}{n}.
\]
The remainder term is of the same order. Since \(\tau_a(c)\ge 0\) (by construction as a variance reduction), dropping \(-\tau_a(c)\) yields an upper bound, proving the claim.
\end{proof}

Theorem~\ref{thm:SSS-mitigates-rev} shows that smoothing replaces the hard-edge blowup \(1/\{n\,c(1-c)\}\) by a tunable bound \(O(a/n)\). Thus, for \(a=o(n)\), boundary inflation is uniformly negligible. Furthermore, a sufficient condition under which ECP can be avoided by SSS is provided below.

\begin{corollary} %[Sufficient conditions to \emph{avoid} ECP]
\label{cor:SSS-avoid-rev}
Fix \(\varepsilon\in(0,1/2)\). There exists \(A_\varepsilon<\infty\) such that if \(a\le A_\varepsilon\) for all \(n\), then for all sufficiently large \(n\),
\[
\sup_{\,c\in(0,\varepsilon]\cup[1-\varepsilon,1)} \Var\!\big(q_a(c)\big)
\ <\
\sup_{\,c\in[\varepsilon,\,1-\varepsilon]} \Var\!\big(q_a(c)\big).
\]
Hence \(\Pr\!\big(\arg\max_{c\in\mathcal C_n} q_a(c)^2\in[\varepsilon,\,1-\varepsilon]\big)\to 1\).
\end{corollary}

\begin{proof}
For \(c\in(0,\varepsilon]\cup[1-\varepsilon,1)\) on the edge, Lemma~\ref{lem:SSS-expansion-rev} and Lemma~\ref{lem:psi-lower-rev} give
\[
\Var(q_a(c))\ \le\ 1-\tau_a(c)\ +\ C\,\frac{a}{n}\ +\ o\!\Big(\frac{a}{n}\Big),
\qquad
\tau_a(c)=\Theta(1)+O\!\Big(\frac{\log n}{n}\Big),
\]
because \(b_a(c)\{1-b_a(c)\}\asymp 1/a\) on the edge layer, so \(\tau_a(c)=\frac{\E[\Delta_{k,a}(c)]}{b_a(1-b_a)}+O(\log n/n)=\Theta(1)\).

On the interior \([\varepsilon,1-\varepsilon]\), \(\bar\psi_a(c)\ge \varepsilon(1-\varepsilon)-C_1/a\), and \(\kappa_a(c)\) is bounded below by a positive constant away from \(c=1/2\).
Thus there exists \(c_\star\in[\varepsilon,1-\varepsilon]\) such that
\[
\Var(q_a(c_\star))\ \ge\ 1-\tau_a(c_\star)\ +\ \frac{c_*}{n}\ +\ o\!\Big(\frac{1}{n}\Big),
\qquad
\tau_a(c_\star)=O\!\Big(\frac{1}{a}\Big)+O\!\Big(\frac{\log n}{n}\Big).
\]
Choose \(A_\varepsilon>0\) so that for all \(a\le A_\varepsilon\) and all large \(n\),
\(
C\,a/n \le (c_*/2)/n
\)
and
\(
\tau_a(c_\star)\le 1/4
\).
Then
\[
\sup_{\text{edge}}\Var(q_a)\ \le\ 1-\tfrac14+\tfrac{c_*}{2n}+o\!\Big(\tfrac{1}{n}\Big)
\ <\
1-\tau_a(c_\star)+\tfrac{c_*}{n}+o\!\Big(\tfrac{1}{n}\Big)
\ \le\ \sup_{\text{int}}\Var(q_a),
\]
for all sufficiently large \(n\), proving the claim.
\end{proof}

Corollary~\ref{cor:SSS-avoid-rev} prescribes a concrete regime, fixed (or slowly varying) \(a\), in which edge points are strictly suboptimal in variance, so the maximizer of \(q_a(c)^2\) lies in the interior with probability tending to one. Together with Theorem~\ref{thm:SSS-mitigates-rev}, this highlights two complementary smoothing effects that counter ECP: a \emph{capped edge scale} \(1/\{n\,\bar\psi_a(c)\}\) of order \(O(a/n)\), and a \emph{boundary penalty} through \(\tau_a(c)=\Theta(1)\) on the edge layer. In addition, the within–risk–set subtraction
\(
\Delta_{k,a}(c)=Y_k^{-1}\sum_{j\in\mathcal R_k} s_a(Z_j;c)\{1-s_a(Z_j;c)\}
\)
from Lemma~\ref{lem:soft-d-mean-var-final} uniformly lowers the per–time variance relative to the hard case. In combination, these mechanisms blunt the variance spike that otherwise favors end cuts and, for suitably chosen \(a\), eliminate the boundary advantage.

% \afterpage{
% \begin{landscape}
\begin{figure}[h]
\centering
\includegraphics[scale=0.6, angle=0]{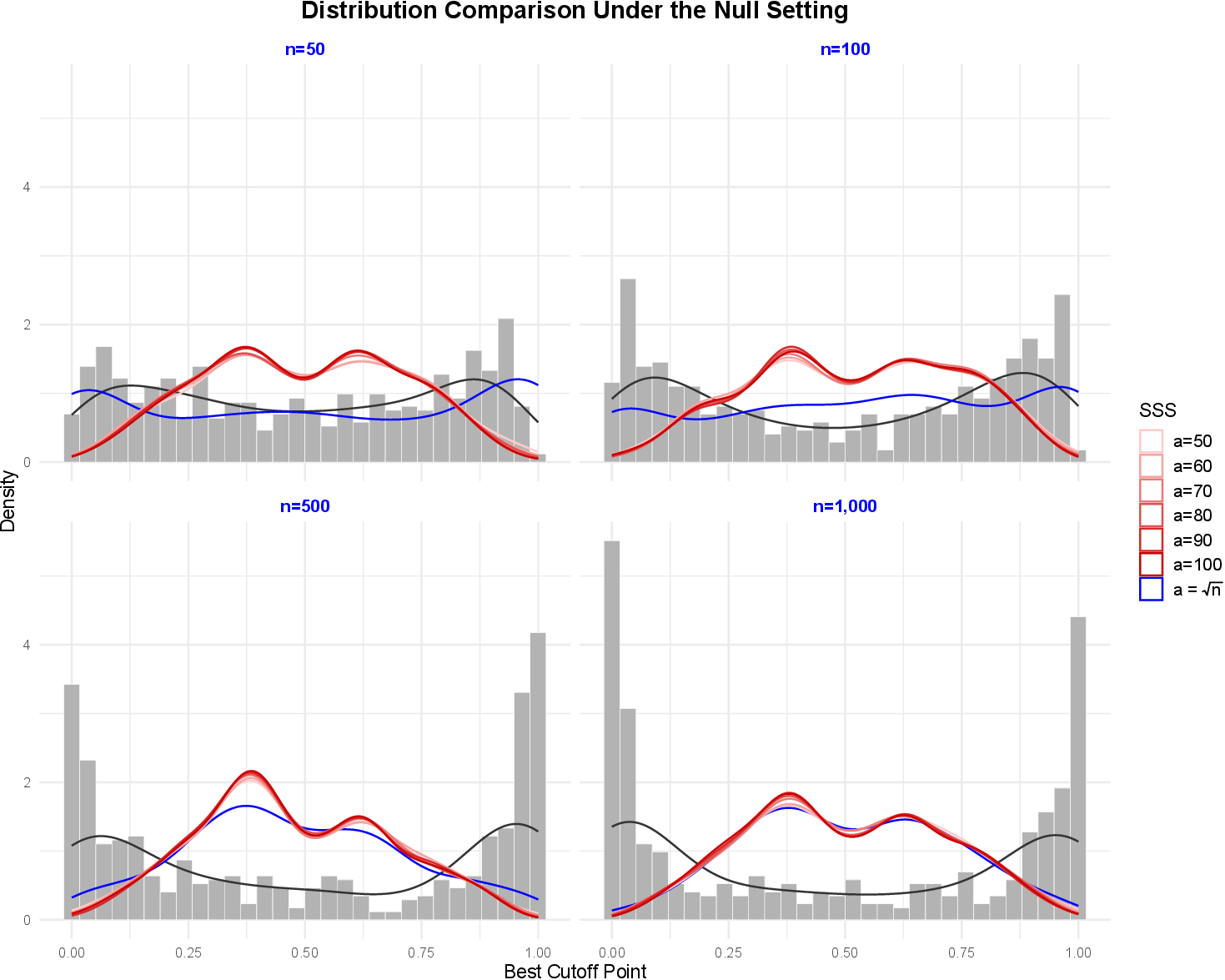} 
 \caption{Empirical distributions of estimated optimal cutoff points under the null setting: greedy search (GS) versus smooth sigmoid surrogate (SSS) methods. The GS method identifies the optimal cutoff by maximizing the log-rank test statistic. The SSS method employs a smooth approximation with scale parameter $a$ varying across ${\sqrt{n}, 50, 60, \ldots, 100}$. Each panel corresponds to a different sample size $n \in \{50, 100, 500, 1000\}$. Histograms represent the GS distribution, while colored density curves depict SSS results with different scale parameters.
\label{fig01-Null}}
\end{figure}
% \end{landscape}
% \clearpage }

\section{Numerical Illustration}
\label{sec-numerical}

For a numerical illustration of the theoretical results, we generate data from the following hazard model:
\begin{equation}
\label{model-split}
\lambda(t) = \exp \!\Big[\beta_0 + \beta_1 \, I(Z \leq c_0) \Big],
\end{equation}
where $\lambda(t)$ denotes the hazard function, the covariate $Z$ is uniformly distributed on $[0,1]$, and the true cutoff point is fixed at $c_0 = 0.5$. The baseline regression coefficient is set to $\beta_0 = 1$, while $\beta_1$ controls the signal strength. Both survival times and censoring times are generated from the same hazard function, which achieves a censoring rate of approximately 50\%.  

For each simulated dataset, both GS and SSS are applied to estimate the optimal cutoff point $\hat{c}$. To fully expose the end-cut preference (ECP), the minimum number of observations allowed in any child node is set to zero in GS, provided that the log-rank statistic is computable. Similarly, the search domain for SSS is taken as the entire interval $[0,1]$. A key tuning parameter of SSS is the shape parameter $a>0$. To examine its impact, we consider $a \in \{50, 60, \ldots, 100\}$, as well as the data-adaptive choice $a=\sqrt{n}$. To further investigate the influence of sample size $n$ on ECP, we consider four settings: $n \in \{50, 100, 500, 1000\}$. For each model configuration, 500 simulation replicates are conducted.

We first examine the null case with $\beta_1 = 0$, where $Z$ has no effect on the hazard. Figure~\ref{fig01-Null} displays the empirical distributions of the estimated cutoff points under GS and SSS, stratified by sample size. The results show that GS suffers from a pronounced ECP problem, which becomes increasingly severe as $n$ grows. This indicates that ECP is essentially an asymptotic phenomenon. With the adaptive choice $a=\sqrt{n}$, SSS substantially mitigates ECP, particularly for large samples; however, when $n=50$, SSS exhibits even stronger ECP than GS. In contrast, fixing $a$ at a relatively large value within $[50,100]$ allows SSS to successfully avoid ECP across all scenarios considered here.

% \afterpage{
% \begin{landscape}
\begin{figure}[h]
\centering
\includegraphics[scale=0.6, angle=0]{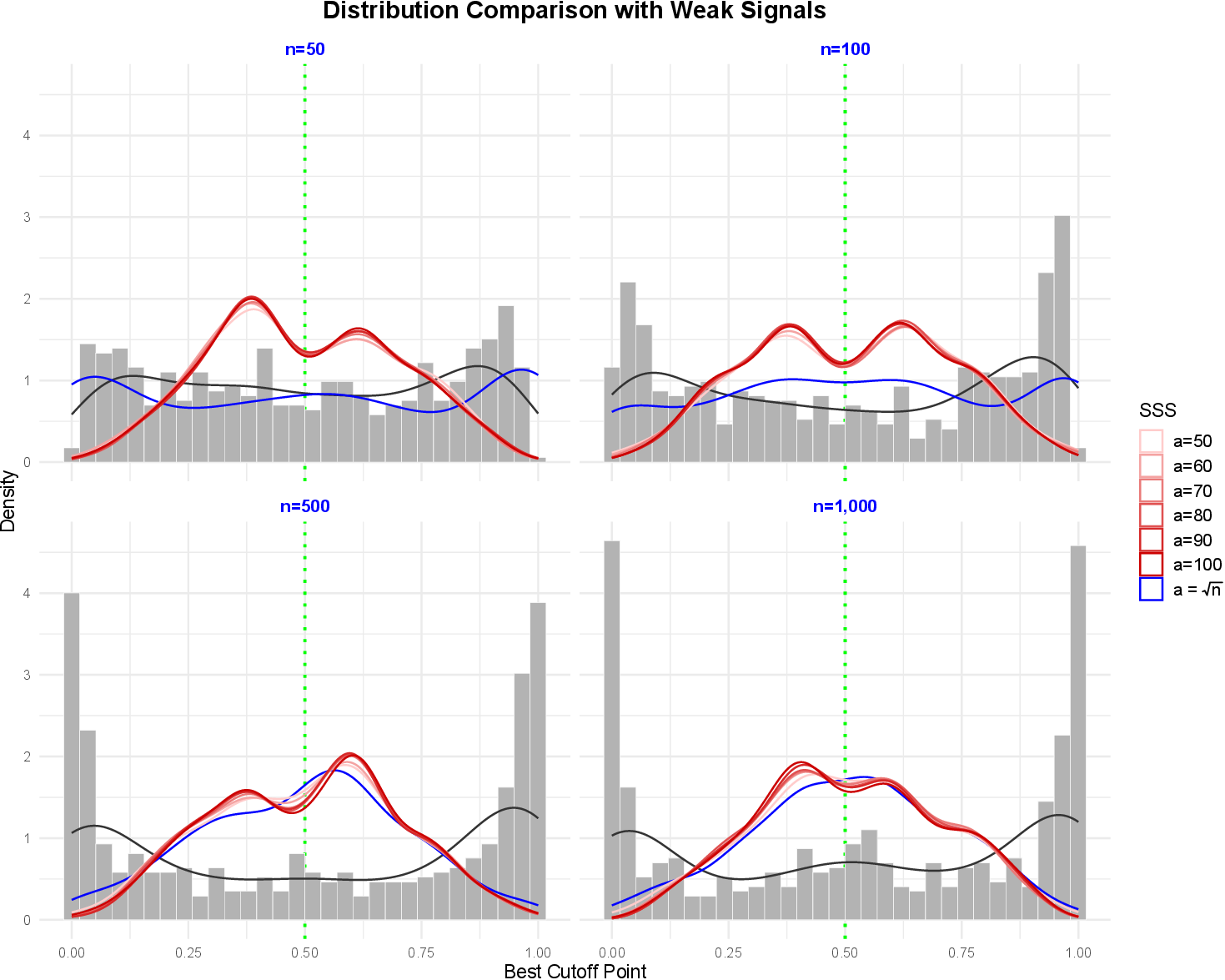} 
 \caption{Empirical distributions of estimated optimal cutoff points with weak signals: greedy search (GS) versus smooth sigmoid surrogate (SSS) methods. The GS method identifies the optimal cutoff by maximizing the log-rank test statistic. The SSS method employs a smooth approximation with scale parameter $a$ varying across ${\sqrt{n}, 50, 60, \ldots, 100}$. Each panel corresponds to a different sample size $n \in \{50, 100, 500, 1000\}$. The vertical dotted line at $c = 0.5$ indicates the true cutoff value. Histograms represent the GS distribution, while colored density curves depict SSS results with different scale parameters.
\label{fig02-Weak}}
\end{figure}
% \end{landscape}
% \clearpage }

We also consider a weak-signal setting with $\beta_1 = -0.1$. The resultant histograms and density estimates are shown in Figure~\ref{fig02-Weak}. The overall pattern is similar. GS continues to suffer from the ECP problem, which can mask weak signals even at large sample sizes. In comparison, SSS either mitigates or avoids ECP, depending on the choice of $a$. From these two studies, setting $a$ at a constant in $[50, 100]$ seems highly advisable.

\section{Discussion}\label{sec-discussion}

In survival trees, we have shown that greedy search (GS) based on the maximized logrank statistic is intrinsically prone to end-cut preference (ECP). The core mechanism is variance inflation at the boundaries: the standardized process satisfies
\[
\Var\{q(c)\}\ =\ 1-\tau\ +\ \frac{\kappa(c)}{n\,c(1-c)}\ +\ o\!\Big(\frac{1}{n\,c(1-c)}\Big),
\]
with \(\kappa(c)\) bounded and strictly positive near the ends. The factor \(1/\{n\,c(1-c)\}\) therefore diverges as \(c\to 0\) or \(1\), tilting the maximizer toward extreme splits. By contrast, the smooth sigmoid surrogate (SSS) replaces the hard indicator by \(s_a(\cdot;c)\), which (i) softens the variance scale from \(c(1-c)\) to \(\bar\psi_a(c)=b_a(c)\{1-b_a(c)\}+O(a^{-1})\) and (ii) introduces the within-risk-set subtraction
\[
\Delta_{k,a}(c)\ =\ \frac{1}{Y_k}\sum_{j\in\mathcal R_k} s_a(Z_j;c)\big\{1-s_a(Z_j;c)\big\},
\]
reducing the per-time conditional variance relative to the hard case. The resulting expansion
\[
\Var\{q_a(c)\}\ =\ 1-\tau_a(c)\ +\ \frac{\kappa_a(c)}{n\,\bar\psi_a(c)}\ +\ o\!\Big(\frac{1}{n\,\bar\psi_a(c)}\Big)
\]
is capped by \(O(a/n)\) uniformly because \(\bar\psi_a(c)\gtrsim 1/a\) near the edges. For \(a=o(n)\) this cap is \(o(1)\), and for fixed \(a\) the interior dominates, thereby mitigating and, under mild choices of \(a\), avoiding ECP.

An implementation of SSS within an entire survival-tree procedure has recently been developed \citep{Zhou:2025}. Beyond addressing ECP, their study evaluates computational aspects (e.g., computing time, numerical stability) and empirical accuracy in recovering population-optimal cutpoints, benchmarking against GS. The reported results indicate substantial improvements in both stability and runtime, as well as more reliable identification of interior cutpoints when the truth is not at the extremes.

Methodologically, SSS brings several advantages to recursive partitioning. First, it converts a discrete, non-smooth split search into a smooth, differentiable optimization in \(c\), enabling gradient-based solvers and alleviating sensitivity to sampling noise. Second, the shape or bandwidth parameter \(a\) provides direct control over the variance scale at the boundaries, replacing the diverging \(1/\{n\,c(1-c)\}\) with a tunable \(O(a/n)\). Third, the within-risk-set subtraction \(\Delta_{k,a}(c)\) uniformly lowers per-time variance, further suppressing artificial preference for end cuts. These features jointly stabilize the split selection, reduce the chance of spurious extreme splits, and improve computational efficiency, making SSS a compelling approach for tree-based modeling.

Finally, our analysis offers a general template for studying ECP beyond the setting with the logrank statistic. Many tree-structured methods rely on two-sample splitting statistics that are asymptotically \(\chi^2(1)\) for a fixed cutpoint. The approach here, (i) recasting the standardized statistic as a mean-zero Gaussian process indexed by the cutpoint, (ii) using extreme-value heuristics to identify the role of the variance function, and (iii) deriving a second-order (Edgeworth) variance expansion to expose boundary terms, extends naturally to other outcomes (classification, regression), data types (longitudinal, time series, functional), and alternative weightings. In these scenarios, a smooth surrogate can again regularize the variance near the edges, capping boundary inflation and thereby mitigating ECP within a unified theoretical framework.

% \newpage
% \hrulefill

%\bibliographystyle{plainnat}


\begin{thebibliography}{99}
\expandafter\ifx\csname
natexlab\endcsname\relax\def\natexlab#1{#1}\fi
\expandafter\ifx\csname url\endcsname\relax
  \def\url#1{\texttt{#1}}\fi
\expandafter\ifx\csname
urlprefix\endcsname\relax\def\urlprefix{URL }\fi

\bibitem[Adler and Taylor(2007)]{Adler:2007}
\item Adler, R.~J.~and Taylor, J.~E.~(2007). \textit{Random Fields and Geometry}. Springer. % (For Gaussian process extrema).

\bibitem[Andersen et al.(1993)]{Andersen:1993}
Andersen, P.~K., Borgan, Ø., Gill, R.~D., and Keiding, N.~(1993).
\newblock {\em Statistical Models Based on Counting Processes}.
\newblock Springer, New York. 

\bibitem[Andersen and Gill(1982)]{Andersen:1982}
Andersen, P.~K.~and Gill, R.~D.~(1982).
Cox's regression model for counting processes: A large sample study.
\emph{The Annals of Statistics}, \textbf{10}(4): 1100--1120.

\bibitem[Breiman et al.(1984)]{Breiman:1984}
Breiman, L., Friedman, J., Olshen, R., and Stone, C.~(1984). \textit{Classification and Regression Trees}. Wadsworth International Group, Belmont, CA.

\bibitem[Bhattacharya and Rao(1986)]{Bhattacharya:1986}
Bhattacharya, R.~N.~and Rao, R.~R.~(1986). \emph{Normal Approximation and Asymptotic Expansions}, Wiley (1986) % for cumulant/Edgeworth techniques.

\bibitem[Bou-Hamad, Larocque, and Ben-Ameur(2011)]{BouHamad:2011}
Bou-Hamad, I., Larocque, D., and Ben-Ameur, H.~(2011) A Review of Survival Trees. \textit{Statistics Surveys}, \textbf{5}: 44--71. 

\bibitem[Cattaneo, Klusowski, and Tian(2024)]{Cattaneo:2024}
Cattaneo, M.~D., Klusowski, J.~M., and Tian, P.~M.~(2024). On the pointwise behavior of recursive partitioning and its implications for heterogeneous causal effect estimation.  \textit{arXiv}, abs/2211.10805. URL: \url{https://arxiv.org/abs/2211.10805}.  % End-cut preference

\bibitem[Ciampi et al.(1986)]{Ciampi:1986}
Ciampi, A., Thiffault, J., Nakache, J.-P. and Asselain, B.~(1986). Stratification by stepwise regression, correspondance analysis and recursive partition: A comparison of three methods of analysis for survival data with
covariates. \textit{Computational Statistics \& Data Analysis}, \textbf{4}: 185--204.   % first use of logrank test stat in survival trees

\bibitem[Hall(1992)]{Hall:1992}
Hall, P.~(1992). \emph{The Bootstrap and Edgeworth Expansion}, Springer (1992)  % Edgeworth expansions and higher-order moment corrections.
    
\bibitem[Ishwaran(2015)]{Ishwaran:2015}
Ishwaran, H.~(2015). The effect of splitting on random forests. \textit{Machine Learning}, \textbf{99}: 75--118. % ECP in RF


\bibitem[Isserlis(1918)]{Isserlis:1918}
Isserlis, L.~(1918). On a formula for the product-moment coefficient of any order of a normal frequency distribution in any number of variables. \textit{Biometrika}, \textbf{12}(1/2): 134--139.

\bibitem[Laurent, Munthe-Kaas, and Vilmart(2025)]{Laurent:2025}
Laurent, A., Munthe-Kaas, H., and Vilmart, G.~(2025). A short proof of Isserlis' theorem. arXiv:2503.01588.

\bibitem[Leadbetter, Lindgren, and Rootzén(1983)]{Leadbetter:1983}
Leadbetter, M.~R., Lindgren, G., and Rootzén, H.~(1983). \textit{Extremes and Related Properties of Random Sequences and Processes}. Springer. % extreme value theory of Gaussian processes

\bibitem[LeBlanc and Crowley(1993)]{LeBlanc:1993}
LeBlanc, M.~and Crowley, J.~(1993). Survival trees by goodness of split. \textit{Journal of the American Statistical Association}, \textbf{88}: 457--467. 

\bibitem[LeBlanc and Crowley(1995)]{LeBlanc:1995}
LeBlanc, M.~and Crowley, J.~(1995). A review of tree–based prognostic
models. \textit{Journal of Cancer Treatment and Research}, \textbf{75}: 113--124.

\bibitem[Lin, Wei, and Ying(1993)]{Lin:1993}
Lin, D.~Y., Wei, L.~J., and Ying, Z.~(1993).
Checking the Cox model with cumulative sums of martingale-based residuals.
\emph{Biometrika}, \textbf{80}(3): 557--572.
% URL: https://doi.org/10.1093/biomet/80.3.557 
  
\bibitem[Mantel(1966)]{Mantel:1966}
Mantel, N.~(1966). Evaluation of survival data and two new rank order statistics arising in its consideration. \textit{Cancer Chemotherapy Reports}, \textbf{50}(3): 163--170.  

\bibitem[Miller and Siegmund(1982)]{Miller:1982}
Miller, R.~and Siegmund, D.~(1982). Maximally selected chi square
statistics. \textit{Biometrics}, \textbf{38}(4): 1011--1016.

\bibitem[Morgan and Sonquist(1963)]{Morgan:1963}
Morgan, J. and Sonquist, J. (1963). Problems in the analysis of survey data and a proposal. \textit{Journal of the American Statistical Association}, \textbf{58}: 415--434.

\iffalse
\bibitem[Oehlert(1992)]{Oehlert:1992}  
Oehlert, G.~W.~(1992). A Note on the Delta Method, \emph{The American Statisticians} \textbf{46}(1):27--29. % concise remarks on second-order delta approximations.
\fi

\bibitem[Peto and Peto(1972)]{Peto:1972}
Peto, R.~and Peto, J.~(1972). Asymptotically efficient rank invariant test procedures. \textit{Journal of the Royal Statistical Society, Series A}, \textbf{135}(2): 185--207.    
    
\bibitem[R Core Team(2025)]{R:2025}
R Core Team (2025). \textit{R: A language and environment for
statistical computing}. R Foundation for Statistical Computing,
Vienna, Austria. URL~\url{https://www.R-project.org/}.    

\bibitem[Segal(1988)]{Segal:1988}
Segal, M.~R.~(1988). Regression trees for censored data. \textit{Biometrics}, \textbf{44}: 35--48.    

\bibitem[Serfling(1980)]{Serfling:1980}
Serfling, R.~J.~(1980). \textit{Approximation Theorems of Mathematical Statistics}, Wiley. % useful for combinatorial moment bounds (sampling without replacement).

\bibitem[Su et al.(2024)]{Su:2024}
Su, X.,  Quaye, G.~E., Wei, Y., Kang, J., Liu, L., Yang, Q., Fan, J., and Levine, R.~A.~(2024). Smooth Sigmoid Surrogate (SSS): An alternative to greedy search in decision trees.  \textit{Mathematics}, \textbf{12}(20), article 3190. % URL: \url{https://doi.org/10.3390/math12203190} 

\bibitem[van der Vaart(1980)]{vanderVaart:1980} 
van der Vaart, A.~W.~(1998). \emph{Asymptotic Statistics}, Cambridge Univ.\ Press (1998). % Ch.~3: Delta method (multivariate).

\bibitem[Wick(1950)]{Wick:1950} 
Wick, G.~C.~(1950). The evaluation of the collision matrix. \textit{Physical Review}, \textbf{80}(2): 268--272. 

\bibitem[Wolter(1985)]{Wolter:1985} 
Wolter, K.~M.~(1985). \textit{Introduction to Variance Estimation}. New York: Springer. % pp. 221–247.

\bibitem[Zhou et al.(2025)]{Zhou:2025} 
Zhou, R., Xie, K., Liu, L., Xu, Z., Ding, J., and Su, X.~(2025). Enhanced survival trees. arXiv:  
\end{thebibliography}
\end{document}